\documentclass{article}

\usepackage{times}
\usepackage{graphicx} 
\usepackage{subfigure} 

\usepackage{natbib}

\usepackage{algorithm}
\usepackage{algorithmic}
\usepackage{amsmath,amssymb,amsthm,bm}
\usepackage{color,graphicx}
\usepackage{natbib,placeins}

\newtheorem{theorem}{Theorem}
\newtheorem{proposition}{Proposition}

\DeclareMathOperator*{\argmin}{arg\,min}
\DeclareMathOperator*{\argmax}{arg\,max}
\newcommand{\appropto}{\mathrel{\vcenter{
  \offinterlineskip\halign{\hfil$##$\cr
    \propto\cr\noalign{\kern2pt}\sim\cr\noalign{\kern-2pt}}}}}

\usepackage[accepted,nohyperref]{icml2017} 

\icmltitlerunning{On the Sampling Problem for Kernel Quadrature}

\begin{document}

\twocolumn[

\icmltitle{On the Sampling Problem for Kernel Quadrature}

\begin{icmlauthorlist}
\icmlauthor{Fran\c{c}ois-Xavier Briol}{war,imp}
\icmlauthor{Chris J. Oates}{newc,ati}
\icmlauthor{Jon Cockayne}{war}
\icmlauthor{Wilson Ye Chen}{uts}
\icmlauthor{Mark Girolami}{imp,ati}
\end{icmlauthorlist}

\icmlaffiliation{war}{University of Warwick, Department of Statistics.}
\icmlaffiliation{imp}{Imperial College London, Department of Mathematics.}
\icmlaffiliation{newc}{Newcastle University, School of Mathematics and Statistics}
\icmlaffiliation{uts}{University of Technology Sydney, School of Mathematical and Physical Sciences}

\icmlaffiliation{ati}{The Alan Turing Institute for Data Science}

\icmlcorrespondingauthor{Fran\c{c}ois-Xavier Briol}{f-x.briol@warwick.ac.uk}

\icmlkeywords{Kernel Quadrature, Numerical Integration, Monte Carlo Methods,Sequential Monte Carlo}

\vskip 0.3in
]



\printAffiliationsAndNotice{} 

\begin{abstract} 
The standard Kernel Quadrature method for numerical integration with random point sets (also called Bayesian Monte Carlo) is known to converge in root mean square error at a rate determined by the ratio $s/d$, where $s$ and $d$ encode the smoothness and dimension of the integrand.
However, an empirical investigation reveals that the rate constant $C$ is highly sensitive to the distribution of the random points.
In contrast to standard Monte Carlo integration, for which optimal importance sampling is well-understood, the sampling distribution that minimises $C$ for Kernel Quadrature does not admit a closed form.
This paper argues that the practical choice of sampling distribution is an important open problem. One solution is considered; a novel automatic approach based on adaptive tempering and sequential Monte Carlo. Empirical results demonstrate a dramatic reduction in integration error of up to 4 orders of magnitude can be achieved with the proposed method.
\end{abstract}

\section{INTRODUCTION}

Consider approximation of the Lebesgue integral
\begin{eqnarray}
\Pi(f) = \int_{\mathcal{X}} f \mathrm{d}\Pi \label{intdef}
\end{eqnarray}
where $\Pi$ is a Borel measure defined over $\mathcal{X} \subseteq \mathbb{R}^d$ and $f$ is Borel measurable.
Define $\mathcal{P}(f)$ to be the set of Borel measures $\Pi'$ such that $f \in L_2(\Pi')$, meaning that $\|f\|_{L_2(\Pi')}^2 = \int_{\mathcal{X}} f^2 \mathrm{d}\Pi' < \infty$, and assume $\Pi \in \mathcal{P}(f)$.
In situations where $\Pi(f)$ does not admit a closed-form, Monte Carlo (MC) methods can be used to estimate the numerical value of Eqn. \ref{intdef}.
A classical research problem in computational statistics is to reduce the MC estimation error in this context, where the integral can, for example, represent an expectation or marginalisation over a random variable of interest.

The default MC estimator comprises of 
$$
\hat{\Pi}_{\text{MC}}(f) = \frac{1}{n} \sum_{j=1}^n f(\bm{x}_j),
$$
where $\bm{x}_j$ are sampled identically and independently (i.i.d.) from $\Pi$.
Then we have a root mean square error (RMSE) bound
$$
\sqrt{\mathbb{E} [\hat{\Pi}_{\text{MC}}(f) - \Pi(f) ]^2} \leq \frac{C_{\text{MC}}(f;\Pi)}{\sqrt{n}},
$$
where $C_{\text{MC}}(f;\Pi) = \text{Std}(f ; \Pi)$ and the expectation is with respect to the joint distribution of the $\{\bm{x}_j\}_{j=1}^n$.
For settings where the Lebesgue density of $\Pi$ is only known up to normalising constant, Markov chain Monte Carlo (MCMC) methods can be used; the rate-constant $C_{\text{MC}}(f;\Pi)$ is then related to the asymptotic variance of $f$ under the Markov chain sample path.

Considerations of computational cost place emphasis on methods to reduce the rate constant $C_{\text{MC}}(f;\Pi)$.
For the MC estimator, this rate constant can be made smaller via importance sampling (IS):
$f \mapsto f \cdot \mathrm{d} \Pi / \mathrm{d} \Pi'$
where an optimal choice $\Pi' \in \mathcal{P}(f \cdot \mathrm{d}\Pi / \mathrm{d} \Pi')$, that minimises $\text{Std}(f \cdot \mathrm{d}\Pi / \mathrm{d} \Pi' ; \Pi')$, is available in explicit closed-form \citep[see][Thm. 3.3.4]{Robert2013}.
However, the RMSE remains asymptotically gated at $O(n^{-1/2})$.

The default Kernel Quadrature (KQ) estimate comprises of
\begin{eqnarray}
\hat{\Pi}(f) = \sum_{j=1}^n w_j f(\bm{x}_j), \label{kerquad}
\end{eqnarray}
where the $\bm{x}_j \sim \Pi'$ are independent (or arise from a Markov chain) and $\text{supp}(\Pi) \subseteq \text{supp}(\Pi')$.
In contrast to MC, the weights $\{w_j\}_{j=1}^n$ in KQ are in general non-uniform, real-valued and depend on $\{\bm{x}_j\}_{j=1}^n$.
The KQ nomenclature derives from the (symmetric, positive-definite) kernel $k : \mathcal{X} \times \mathcal{X} \rightarrow \mathbb{R}$ that is used to construct an interpolant $\hat{f}(\bm{x}) = \sum_{j=1}^n \beta_j k(\bm{x},\bm{x}_j)$ such that $\hat{f}(\bm{x}_j) = f(\bm{x}_j)$ for $j = 1,\dots,n$.
The weights $w_j$ in Eqn. \ref{kerquad} are implicitly defined via the equation $\hat{\Pi}(f) = \int_{\mathcal{X}} \hat{f} \mathrm{d} \Pi$.
The KQ estimator is identical to the posterior mean in Bayesian Monte Carlo \citep{OHagan1991,Rasmussen2002}, and its relationship with classical numerical quadrature rules has been studied \citep{Diaconis1988,Sarkka2015}. 

Under regularity conditions, \citet{Briol2016} established the following RMSE bound for KQ:
\begin{eqnarray*}
\sqrt{\mathbb{E} [\hat{\Pi}(f) - \Pi(f) ]^2} \leq \frac{C(f; \Pi')}{n^{s/d - \epsilon}}, \quad (s > d/2) \label{eq:BQerror}
\end{eqnarray*}
where both the integrand $f$ and each argument of the kernel $k$ admit continuous mixed weak derivatives of order $s$ and $\epsilon > 0$ can be arbitrarily small.
An information-theoretic lower bound on the RMSE is $O(n^{-s/d-1/2})$ \citep{Bakhvalov1959}.
The faster convergence of the RMSE, relative to MC, can lead to improved precision in applications.
Akin to IS, the samples $\{\bm{x}_j\}_{j=1}^n$ need not be draws from $\Pi$ in order for KQ to provide consistent estimation (since $\Pi$ is encoded in the weights $w_j$).
Importantly, KQ can be viewed as post-processing of MC samples; the kernel $k$ can be reverse-engineered (e.g. via cross-validation) and does not need to be specified up-front.

One notable disadvantage of KQ methods is that little is known about how the rate constant $C(f; \Pi')$ depends on the choice of sampling distribution $\Pi'$.
In contrast to IS, no general closed-form expression has been established for an optimal distribution $\Pi'$ for KQ (the technical meaning of `optimal' is defined below).
Moreover, limited practical guidance is available on the selection of the sampling distribution \citep[an exception is][as explained in Sec. \ref{subsec:estres}]{Bach2015} and in applications it is usual to take $\Pi' = \Pi$.
\\
This choice is convenient but leads to estimators that are not efficient, as we demonstrate in dramatic empirical examples in Sec. \ref{subsec:motivation}.

The main contributions of this paper are twofold. 
First, we formalise the problem of optimal sampling for KQ as an important and open challenge in computational statistics.
To be precise, our target is an optimal sampling distribution for KQ, defined as
\begin{eqnarray}
\Pi^* \in  \argmin_{\Pi'} \sup_{f \in \mathcal{F}} \sqrt{\mathbb{E} [\hat{\Pi}(f) - \Pi(f) ]^2} . \label{objective function}
\end{eqnarray}
for some functional class $\mathcal{F}$ to be specified. In general a (possibly non-unique) optimal $\Pi^*$ will depend on $\mathcal{F}$ and, unlike for IS, also on the kernel $k$ and the number of samples $n$.

Second, we propose a novel and automatic method for selection of $\Pi'$ that is rooted in approximation of the unavailable $\Pi^*$.
In brief, our method considers candidate sampling distributions of the form $\Pi' = \Pi_0^{1-t} \Pi^t$ for $t \in [0,1]$ and $\Pi_0$ a reference distribution on $\mathcal{X}$.
The exponent $t$ is chosen such that $\Pi'$ minimises an empirical upper bound on the RMSE.
The overall approach is facilitated with an efficient sequential MC (SMC) sampler and called $\texttt{SMC-KQ}$.
In particular, the approach (i) provides practical guidance for selection of $\Pi'$ for KQ, (ii) offers robustness to kernel mis-specification, and (iii) extends recent work on computing posterior expectations with kernels obtained using Stein's method \citep{Oates2017}.

The paper proceeds as follows:
Empirical results in Sec. \ref{sec:background} reveal that the RMSE for KQ is highly sensitive to the choice of $\Pi'$.
The proposed approach to selection of $\Pi'$ is contained in Sec. \ref{sec:methods}.
Numerical experiments, presented in Sec. \ref{sec:results}, demonstrate that dramatic reductions in integration error (up to 4 orders of magnitude) can be achieved with $\texttt{SMC-KQ}$.
Lastly, a discussion is provided in Sec. \ref{sec:discussion}.

\section{BACKGROUND} \label{sec:background}

This section presents an overview of KQ (Sec. \ref{subsec:overview} and \ref{no SBQ}), empirical (Secs. \ref{subsec:motivation}) and theoretical (Sec. \ref{subsec:estres}) results on the choice of sampling distribution, and discusses kernel learning for KQ (Sec. \ref{sec:kernel sen}).

\subsection{Overview of Kernel Quadrature} \label{subsec:overview}

We now proceed to describe KQ:
Recall the approximation $\hat{f}$ to $f$; an explicit form for the coefficients $\beta_j$ is given as $\bm{\beta} = \bm{\mathrm{K}}^{-1} \bm{\mathrm{f}}$, where $\mathrm{K}_{i,j} = k(\bm{x}_i,\bm{x}_j)$ and $\mathrm{f}_j = f(\bm{x}_j)$.
It is assumed that $\bm{\mathrm{K}}^{-1}$ exists almost surely; for non-degenerate kernels, this corresponds to $\Pi$ having no atoms.
From the above definition of KQ,
\begin{eqnarray*}
\hat{\Pi}(f) = \sum_{j=1}^n \beta_j \int_{\mathcal{X}} k(\bm{x},\bm{x}_j) \Pi(\mathrm{d} \bm{x}) .
\end{eqnarray*}
Defining $z_j = \int_{\mathcal{X}} k(\cdot,\bm{x}_j) \mathrm{d}\Pi$ leads to the estimate in Eqn. \ref{kerquad} with weights $\bm{w} = \bm{\mathrm{K}}^{-1} \bm{z}$.
Pairs $(\Pi,k)$ for which the $z_j$ have closed form are reported in Table 1 of \citet{Briol2016}.
Computation of these weights incurs a computational cost of at most $O(n^3)$ and can be justified when either (i) evaluation of $f$ forms the computational bottleneck, or (ii) the gain in estimator precision (as a function in $n$) dominates this cost (i.e. whenever $s/d > 3 + 1/2$).

Notable contributions on KQ include \citet{Diaconis1988,OHagan1991,Rasmussen2002} who introduced the method and \citet{Huszar2012,Osborne2012active,Osborne2012,Gunter2014,Bach2015,Briol2015,Briol2016,Sarkka2015,Kanagawa2016,Liu2016} who provided consequent methodological extensions. KQ has been applied to a wide range of problems including probabilistic ODE solvers \citep{Kersting2016}, reinforcement learning \citep{Paul2016}, filtering \citep{Pruher2015} and design of experiments \citep{Ma2014}.

Several characterisations of the KQ estimator are known and detailed below.
Let $\mathcal{H}$ denote the Hilbert space characterised by the reproducing kernel $k$, and denote its norm as $\|\cdot\|_{\mathcal{H}}$ \citep{Berlinet2011}.
Then we have the following:
(a) The function $\hat{f}$ is the minimiser of $\|g\|_{\mathcal{H}}$ over $g \in \mathcal{H}$ subject to $g(\bm{x}_j) = f(\bm{x}_j)$ for all $j = 1,\dots,n$.
(b) The function $\hat{f}$ is the posterior mean for $f$ under the Gaussian process prior $f \sim \text{GP}(0,k)$ conditioned on data $\mathbf{f}$ and $\hat{\Pi}(f)$ is the mean of the implied posterior marginal over $\Pi[f]$.
(c) The weights $\bm{w}$ are characterised as the minimiser over $\bm{\gamma} \in \mathbb{R}^n$ of 
$$
e_n(\bm{\gamma};\{\bm{x}_j\}_{j=1}^n) = \sup_{\|f\|_{\mathcal{H}} = 1} \Bigg|\sum_{j=1}^n \gamma_j f(\bm{x}_j) - \Pi(f)\Bigg| ,
$$
the maximal error in the unit ball of $\mathcal{H}$.
These characterisations connect KQ to (a) non-parametric regression, (b) probabilistic integration and (c) quasi-Monte Carlo (QMC) methods \citep{Dick2010}.
The scattered data approximation literature \citep{Sommariva2006} and the numerical analysis literature \citep[where KQ is known as the `empirical interpolation method';][]{Eftang2012,Kristoffersen2013} can also be connected to KQ.
However, our search of all of these literatures did not yield guidance on the optimal selection of the sampling distribution $\Pi'$ (with the exception of \citet{Bach2015} reported in Sec. \ref{subsec:estres}).

\subsection{Over-Reliance on the Kernel} \label{no SBQ}

In \citet{Osborne2012active,Huszar2012,Gunter2014,Briol2015}, the selection of $\bm{x}_n$ was approached as a greedy optimisation problem, wherein the maximal integration error $e_n(\bm{w};\{\bm{x}_j\}_{j=1}^n)$ was minimised, given the location of the previous $\{\bm{x}_j\}_{j=1}^{n-1}$.
This approach has demonstrated considerable success in applications.
However, the error criterion $e_n$ is strongly dependant on the choice of kernel $k$ and the sequential optimisation approach is vulnerable to kernel misspecification.
In particular, if the intrinsic length scale of $k$ is ``too small'' then the $\{\bm{x}_j\}_{j=1}^n$ all cluster around the mode of $\Pi$, leading to poor integral estimation (see Fig. \ref{SBQ} in the Appendix).
Related work on sub-sample selection, such as leverage scores \citep{Bach2013}, can also be non-robust to mis-specified kernels.
The partial solution of online kernel learning requires a sufficient number $n$ of data and is not always practicable in small-$n$ regimes that motivate KQ.

This paper considers sampling methods as a robust alternative to optimisation methods.
Although our method also makes use of $k$ to select $\Pi'$, it reverts to $\Pi' = \Pi$ in the limit as the length scale of $k$ is made small.
In this sense, sampling offers more robustness to kernel mis-specification than optimisation methods, at the expense of a possible (non-asymptotic) decrease in precision in the case of a well-specified kernel.
This line of research is thus complementary to existing work.
However, we emphasise that robustness is an important consideration for general applications of KQ in which kernel specification may be a non-trivial task.

\subsection{Sensitivity to the Sampling Distribution} \label{subsec:motivation}

To date, we are not aware of a clear demonstration of the acute dependence of the performance of the KQ estimator on the choice of distribution $\Pi'$.
It is therefore important to illustrate this phenomenon in order to build intuition.

Consider the toy problem with state space $\mathcal{X} = \mathbb{R}$, target distribution $\Pi = \text{N}(0,1)$, a single test function $f(x) = 1 + \sin(2\pi x)$ and kernel $k(x,x') = \exp(-(x-x')^2)$.
For this problem, consider a range of sampling distributions of the form $\Pi' = \text{N}(0,\sigma^2)$ for $\sigma \in (0,\infty)$.
Fig. \ref{fig:sample sensitive} plots 
$$
\hat{R}_{n,\sigma} = \sqrt{\frac{1}{M} \sum_{m=1}^M (\hat{\Pi}_{n,m,\sigma}(f) - \Pi(f) )^2 },
$$
an empirical estimate for the RMSE where $\hat{\Pi}_{n,m,\sigma}(f)$ is the $m$th of $M$ independent KQ estimates for $\Pi(f)$ based on $n$ samples drawn from the distribution $\Pi'$ with standard deviation $\sigma$ ($M=1000$).
In this case $\Pi(f) = 1$ is available in closed-form.
It is seen that the `obvious' choice of $\sigma = 1$, i.e. $\Pi' = \Pi$, is sub-optimal.
The intuition here is that `extreme' samples $\bm{x}_i$ from the tails of $\Pi$ are rather informative for building the interpolant $\hat{f}$ underlying KQ; we should therefore over-sample these values via a heavier-tailed $\Pi'$.
The same intuition is used for column sampling and to construct leverage scores \citep{Mahoney2011,Drineas2012}.

\begin{figure}[t!]
\includegraphics[width = \columnwidth]{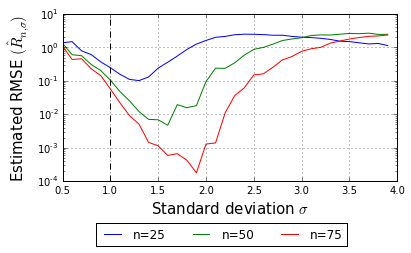}
\caption{The performance of kernel quadrature is sensitive to the choice of sampling distribution.
Here the test function was $f(x) = 1 + \sin(2 \pi x)$, the target measure was $\text{N}(0,1)$, while $n$ samples were generated from $\text{N}(0,\sigma^2)$.
The kernel $k(x,x') = \exp(-(x-x')^2)$ was used.
Notice that the values of $\sigma$ that minimise the root mean square error (RMSE) are uniformly greater than $\sigma=1$ (dashed line) and depend on the number $n$ of samples in general.}
\label{fig:sample sensitive}
\end{figure}

\subsection{Established Results} \label{subsec:estres}

Here we recall the main convergence results to-date on KQ and discuss how these relate to choices of sampling distribution.
To reduce the level of detail below, we make several assumptions at the outset:

{\bf Assumption on the domain:}
The domain $\mathcal{X}$ will either be $\mathbb{R}^d$ itself or a compact subset of $\mathbb{R}^d$ that satisfies an `interior cone condition', meaning that there exists an angle $\theta \in (0,\pi/ 2)$ and a radius $r > 0$ such that for every $\bm{x} \in \mathcal{X}$ there exists $\|\bm{\xi}\|_2 = 1$ such that the cone $\{\bm{x} + \lambda\bm{y} \; : \; \bm{y} \in \mathbb{R}^d, \; \|\bm{y}\|_2 = 1, \; \bm{y}^T\bm{\xi} \geq \cos\theta , \; \lambda \in [0,r]\}$ is contained in $\mathcal{X}$ \citep[see][for background]{Wendland2004}.

{\bf Assumption on the kernel:}
Consider the integral operator $\Sigma: L_2(\Pi) \rightarrow L_2(\Pi)$, with $(\Sigma f)(\bm{x})$ defined as the Bochner integral $\int_{\mathcal{X}} f(\bm{x}') k(\bm{x},\bm{x}') \Pi(\mathrm{d}\bm{x}')$.
Assume that $\int_{\mathcal{X}} k(\bm{x},\bm{x}) \Pi(\mathrm{d}\bm{x}) < \infty$, so that $\Sigma$ is self-adjoint, positive semi-definite and trace-class \citep{Simon1979}.
Then, from an extension of Mercer's theorem \citep{Koenig1986} we have a decomposition $k(\bm{x},\bm{x}') = \sum_{m=1}^\infty \mu_m e_m(\bm{x}) e_m(\bm{x}')$, where $\mu_m$ and $e_m(x)$ are the eigenvalues and eigenfunctions of $\Sigma$.
Further assume that $\mathcal{H}$ is dense in $L_2(\Pi)$.

The first result is adapted and extended from Thm. 1 in \citet{Oates2016}.

\begin{theorem} \label{theo:BMC}
Assume that $\Pi'$ admits a density $\pi'$ defined on a compact domain $\mathcal{X}$.
Assume that $\pi' > c$ for some $c > 0$.
Let $\bm{x}_1,\dots,\bm{x}_m$ be fixed and define the Euclidean fill distance
$$
h_m = \sup_{\bm{x} \in \mathcal{X}} \min_{j = 1,\dots,m} \|\bm{x} - \bm{x}_j \|_2.
$$
Let $\bm{x}_{m+1}, \dots , \bm{x}_n$ be independent draws from $\Pi'$. 
Assume $k$ gives rise to a Sobolev space $\mathbb{H}_s(\Pi)$.
Then there exists $h_0 > 0$ such that, for $h_m < h_0$,
\begin{eqnarray*}
\sqrt{ \mathbb{E} [ \hat{\Pi}(f) - \Pi(f) ]^2 } \leq C(f) n^{-s/d + \epsilon}
\end{eqnarray*}
for all $\epsilon > 0$.
Here $C(f) = c_{k,\Pi',\epsilon} \|f\|_{\mathcal{H}}$ for some constant $0 < c_{k,\Pi',\epsilon} < \infty$ independent of $n$ and $f$.
\end{theorem}
All proofs are reserved for the Appendix. 
The main contribution of Thm. \ref{theo:BMC} is to establish a convergence rate for KQ when using importance sampling distributions.
A similar result appeared in Thm. 1 of \citet{Briol2016} for samples from $\Pi$ (see the Appendix) and was extended to MCMC samples in \citet{Oates2016}.
An extension to the case of a mis-specified kernel was considered in \citet{Kanagawa2016}.
However a limitation of this direction of research is that it does not address the question of how to select $\Pi'$.

\begin{figure}[t!]
\includegraphics[width = \columnwidth]{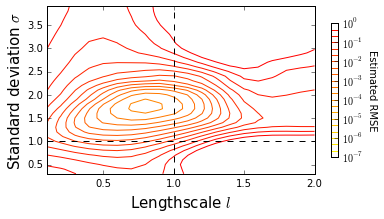}
\caption{The performance of kernel quadrature is sensitive to the choice of kernel.
Here the same set-up as Fig. \ref{fig:sample sensitive} was used with $n = 75$.
The kernel $k(x,x') = \exp(-(x-x')^2/\ell^2)$ was used for various choices of parameter $\ell \in (0,\infty)$.
The root mean square error (RMSE) is sensitive to choice of $\ell$ for all choices of $\sigma$, suggesting that online kernel learning could be used to improve over the default choice of $\ell = 1$ and $\sigma=1$ (dashed lines).}
\label{fig:kernel sensitive}
\end{figure}

The second result that we present is a consequence of the recent work of \citet{Bach2015}, who considered a particular choice of $\Pi' = \Pi_{\text{B}}$, depending on a fixed $\lambda > 0$, via the density 
$
\pi_{\text{B}}(\bm{x}; \lambda) \propto \sum_{m=1}^\infty \frac{\mu_m}{\mu_m + \lambda} e_m^2(\bm{x}).
$
The following is adapted from Prop. 1 in \citet{Bach2015}:
\begin{theorem} \label{bach result}
Let $\bm{x}_1,\dots,\bm{x}_n \sim \Pi_{\text{B}}$ be independent and $\lambda > 0$.
For $\delta \in (0,1)$ and 
$
n \geq 5 d(\lambda) \log \frac{16 d(\lambda) }{ \delta}$, $d(\lambda) = \sum_{m=1}^\infty \frac{\mu_m}{\mu_m + \lambda},
$
we have that
\begin{equation*}
|\hat{\Pi}(f) - \Pi(f)| \leq 2 \lambda^{1/2} \|f\|_{\mathcal{H}} ,
\end{equation*}
with probability greater than $1-\delta$.
\end{theorem}
Some remarks are in order:
(i) \citet[][Prop. 3]{Bach2015} showed that, for $\Pi_\text{B}$, integration error scales at an optimal rate in $n$ up to logarithmic terms and, after $n$ samples, is of size $\sqrt{\mu_n}$. 
(ii) The distribution $\Pi_{\text{B}}$ is obtained from minimising an upper bound on the integration error, rather than the error itself.
It is unclear to us how well $\Pi_{\text{B}}$ approximates an optimal sampling distribution for KQ.
(iii) In general $\Pi_{\text{B}}$ is hard to compute.
For the specific case $\mathcal{X} = [0,1]^d$, $\mathcal{H}$ equal to $\mathbb{H}_s(\Pi)$ and $\Pi$ uniform, the distribution $\Pi_{\text{B}}$ is also uniform (and hence independent of $n$; see Sec. 4.4 of \citet{Bach2015}).
However, even for the simple example of Sec. \ref{subsec:motivation}, $\Pi_{\text{B}}$ does not appear to have a closed form (details in Appendix).
An approximation scheme was proposed in Sec. 4.2 of \citet{Bach2015} but the error of this scheme was not studied.

Optimal sampling for approximation in $\|\cdot\|_{L_2(\Pi)}$ with weighted least squares (not in the kernel setting) was considered in \citet{Hampton2015,Cohen2016}.

\subsection{Goals} \label{sec:kernel sen}

Our first goal was to formalise the sampling problem for KQ; this is now completed.
Our second goal was to develop a novel automatic approach to selection of $\Pi'$, called $\texttt{SMC-KQ}$; full details are provided in Sec.  \ref{sec:methods}. 

Also, observe that the integrand $f$ will in general belong to an infinitude of Hilbert spaces, while for KQ a single kernel $k$ must be selected.
This choice will affect the performance of the KQ estimator; for example, in Fig. \ref{fig:kernel sensitive}, the problem of Sec. \ref{subsec:motivation} was reconsidered based on a class of kernels $k(x,x') = \exp(-(x-x')^2/\ell^2)$ parametrised by $\ell \in (0,\infty)$.
Results showed that, for all choices of $\sigma$ parameter, the RMSE of KQ is sensitive to choice of $\ell$.
In particular, the default choice of $\ell = 1$ is not optimal.
For this reason, an extension that includes kernel learning, called $\texttt{SMC-KQ-KL}$, is proposed in Sec. \ref{sec:methods}.

\section{METHODS} \label{sec:methods}

In this section the $\texttt{SMC-KQ}$ and $\texttt{SMC-KQ-KL}$ methods are presented.
Our aim is to explain in detail the main components (\texttt{SMC}, \texttt{temp}, \texttt{crit}) of Alg. \ref{SMCKQ}.
To this end, Secs. \ref{subsec:thermo} and \ref{subsec:SMC} set up our SMC sampler to target tempered distributions, while Sec. \ref{subsec:optimal t} presents a heuristic for the choice of temperature schedule.
Sec. \ref{sec:KL} extends the approach to kernel learning and Sec. \ref{subsec:stopping} proposes a novel criterion to determine when a desired error tolerance is reached.

\subsection{Thermodynamic Ansatz} \label{subsec:thermo}

To begin, consider $f$, $k$ and $n$ as fixed.
The following ansatz is central to our proposed $\texttt{SMC-KQ}$ method: 
An optimal distribution $\Pi^*$ (in the sense of Eqn. \ref{objective function}) can be well-approximated by a distribution of the form
\begin{eqnarray}
\Pi_t = \Pi_0^{1-t} \Pi^t , \quad t \in [0,1] \label{ansatz}
\end{eqnarray}
for a specific (but unknown) `inverse temperature' parameter $t = t^*$.
Here $\Pi_0$ is a reference distribution to be specified and which should be chosen to be un-informative in practice.
It is assumed that all $\Pi_t$ exist (i.e. can be normalised).
The motivation for this ansatz stems from Sec. \ref{subsec:motivation}, where $\Pi = \text{N}(0,1)$ and $\Pi_t = \text{N}(0,\sigma^2)$ can be cast in this form with $t = \sigma^{-1}$ and $\Pi_0$ an (improper) uniform distribution on $\mathbb{R}$.
In general, tempering generates a class of distributions which over-represent extreme events relative to $\Pi$ (i.e. have heavier tails).
This property has the potential to improve performance for KQ, as demonstrated in Sec. \ref{subsec:motivation}.

The ansatz of Eqn. \ref{ansatz} reduces the non-parametric sampling problem for KQ to the one-dimensional parametric problem of selecting a suitable $t \in [0,1]$.
The problem can be further simplified by focusing on a discrete temperature ladder $\{t_i\}_{i=0}^T$ such that $t_0 = 0$, $t_i < t_{i+1}$ and $t_T = 1$.
Discussion of the choice of ladder is deferred to Sec. \ref{subsec:optimal t}.
This reduced problem, where we seek an optimal index $i^* \in \{0,\dots,T\}$, is still non-trivial as no closed-form expression is available for the RMSE at each candidate $t_i$.
To overcome this {\it impasse} a novel approach to estimate the RMSE is presented in Sec. \ref{subsec:stopping}.

\subsection{Convex Ansatz (\texttt{SMC})}  \label{subsec:SMC}

The proposed $\texttt{SMC-KQ}$ algorithm requires a second ansatz, namely that the RMSE is convex in $t$ and possesses a global minimum in the range $t \in (0,1)$. 
This second ansatz (borne out in numerical results in Fig. \ref{fig:sample sensitive}) motivates an algorithm that begins at $t_0=0$ and tracks the RMSE until an increase is detected, say at $t_i$; at which point the index $i^* = i - 1$ is taken for KQ.

To realise such an algorithm, this paper exploited SMC methods \citep{Chopin2002,Moral2006}. 
Here, a particle approximation $\{(w_j,\bm{x}_j)\}_{j=1}^N$ to $\Pi_{t_0}$ is first obtained where $\bm{x}_j$ are independent draws from $\Pi_0$, $w_j = N^{-1}$ and $N \gg n$.
Then, at iteration $i$, the particle approximation to $\Pi_{t_{i-1}}$ is re-weighted, re-sampled and subject to a Markov transition, to deliver a particle approximation $\{(w_j' , \bm{x}_j')\}_{j=1}^N$ to $\Pi_{t_i}$.
This `re-sample-move' algorithm, denoted $\texttt{SMC}$, is standard but, for completeness, pseudo-code is provided as Alg. \ref{SMC} in the Appendix.

At iteration $i$, a subset of size $n$ is drawn from the unique\footnote{This ensures that kernel matrices have full rank. It does \emph{not} introduce bias into KQ, since in general $\Pi'$ need not equal $\Pi$. However, to keep notation clear, we do not make this operation explicit.} elements in $\{\bm{x}_j'\}_{j=1}^N$, from the particle approximation to $\Pi_{t_i}$, and proposed for use in KQ.
A criterion $\texttt{crit}$, defined in Sec. \ref{subsec:stopping}, is used to determine whether the resultant KQ error has increased relative to $\Pi_{t_{i-1}}$.
If this is the case, then the distribution $\Pi_{t_{i-1}}$ from the previous iteration is taken for use in KQ.
Otherwise the algorithm proceeds to $t_{i+1}$ and the process repeats.
In the degenerate case where the RMSE has a minimum at $t_T$, the algorithm defaults to standard KQ with $\Pi' = \Pi$.

Both ansatz of the $\texttt{SMC-KQ}$ algorithm are justified through the strong empirical results presented in Sec. \ref{sec:results}.

\subsection{Choice of Temperature Schedule (\texttt{temp})} \label{subsec:optimal t}

The choice of temperature schedule $\{t_i\}_{i=0}^T$ influences several aspects of $\texttt{SMC-KQ}$:
(i) The SMC approximation to $\Pi_{t_i}$ is governed by the ``distance" (in some appropriate metric) between $\Pi_{t_{i-1}}$ and $\Pi_{t_i}$.
(ii) The speed at which the minimum $t^*$ can be reached is linear in the number of temperatures between $0$ and $t^*$.
(iii) The precision of KQ depends on the approximation $t^* \approx t_{i^*}$.
Factors (i,iii) motivate the use of a fine schedule with $T$ large, while (ii) motivates a coarse schedule with $T$ small.

For this work, a temperature schedule was used that is well suited to both (i) and (ii), while a strict constraint $t_i - t_{i-1} \leq \Delta$ was imposed on the grid spacing to acknowledge (iii).
The specific schedule used in this work was determined based on the conditional effective sample size of the current particle population, as proposed in the recent work of \citet{Zhou2016}.
Full details are presented in Algs. \ref{cess} and \ref{temp} in the Appendix.

\subsection{Kernel Learning} \label{sec:KL}

In Sec. \ref{sec:kernel sen} we demonstrated the benefit of kernel learning for KQ.
From the Gaussian process characterisation of KQ from Sec. \ref{subsec:overview}, it follows that kernel parameters $\theta$ can be estimated, conditional on a vector of function evaluations $\bm{\mathrm{f}}$, via maximum marginal likelihood:
\begin{eqnarray*}
\theta' & \gets & \argmax_\theta p(\bm{\mathrm{f}} | \theta)\; = \; \argmin_\theta \bm{\mathrm{f}}^\top \bm{\mathrm{K}}_\theta^{-1} \bm{\mathrm{f}} + \log |\bm{\mathrm{K}}_\theta|.
\end{eqnarray*}
In $\texttt{SMC-KQ-KL}$, the function evaluations $\bm{\mathrm{f}}$ are obtained at the first\footnote{This is a notational convention and is without loss of generality. In this paper these states were a random sample (without replacement) of size $n$, though stratified sampling among the $N$ states could be used. More sophisticated alternatives that also involve the kernel $k$, such as leverage scores, were {\bf not} considered, since in general these (i) introduce a vulnerability to mis-specified kernels and (ii) require manipulation of a $N \times N$ kernel matrix \citep{Patel2015}.} $n$ (of $N$) states $\{\bm{x}_j\}_{j=1}^n$ and the parameters $\theta$ are updated in each iteration of the SMC.
This demands repeated function evaluation; this burden can be reduced with less frequent parameter updates and caching of all previous function evaluations.
The experiments in Sec. \ref{sec:results} assessed both $\texttt{SMC-KQ}$ and $\texttt{SMC-KQ-KL}$ in terms of precision per \emph{total} number of function evaluations, so that the additional cost of kernel learning was taken into account.


\begin{algorithm}[t!]
\caption{SMC Algorithm for KQ}
\label{SMCKQ}
\begin{algorithmic}
\STATE {\bf function} $\texttt{SMC-KQ}(f, \Pi, k, \Pi_0, \rho , n, N)$
\STATE {\bf input} $f$ (integrand)
\STATE {\bf input} $\Pi$ (target disn.) 
\STATE {\bf input} $k$ (kernel)
\STATE {\bf input} $\Pi_0$ (reference disn.)
\STATE {\bf input} $\rho$ (re-sample threshold) 
\STATE {\bf input} $n$ (num. func. evaluations)
\STATE {\bf input} $N$ (num. particles)
\STATE $i \gets 0$; $t_i \gets 0$; $R_{\min} \gets \infty$
\STATE $\bm{x}_j' \sim \Pi_0$ (initialise states $\forall j \in 1:N$)
\STATE $w_j' \gets N^{-1}$ (initialise weights $\forall j \in 1:N$)
\STATE $R \gets \texttt{crit} (\Pi , k , \{\bm{x}_j'\}_{j=1}^N)$ (est'd error)
\WHILE{$\texttt{test}(R < R_{\min})$ and $t_i < 1$}
\STATE $i \gets i + 1$; $R_{\min} \gets R$
\STATE $\{(w_j , \bm{x}_j)\}_{j=1}^N \gets \{(w_j' , \bm{x}_j')\}_{j=1}^N$
\STATE $t_i \gets \texttt{temp}(\{(w_j , \bm{x}_j)\}_{j=1}^N , t_{i-1},\rho)$ (next temp.)
\STATE $\{(w_j' , \bm{x}_j')\}_{j=1}^N \gets \texttt{SMC}(\{(w_j,\bm{x}_j)\}_{j=1}^N , t_i , t_{i-1} , \rho)$ (next particle approx.)
\STATE $R \gets \texttt{crit} (\Pi , k , \{\bm{x}_j'\}_{j=1}^N)$ (est'd error)
\ENDWHILE
\STATE $\mathrm{f}_j \gets f(\bm{x}_j)$ (function eval. $\forall j \in 1:n$)
\STATE $z_j \gets \int_{\mathcal{X}} k(\cdot , \bm{x}_j) \mathrm{d}\Pi$ (kernel mean eval. $\forall j \in 1:n$)
\STATE $\mathrm{K}_{j,j'} \gets k(\bm{x}_j , \bm{x}_{j'})$ (kernel eval. $\forall j,j' \in 1:n$)
\STATE $\hat{\Pi}(f) \gets \bm{z}^\top \bm{\mathrm{K}}^{-1} \bm{\mathrm{f}}$ (eval. KQ estimator) 
\STATE {\bf return} $\hat{\Pi}(f)$
\end{algorithmic} 
\end{algorithm}


\subsection{Termination Criterion (\texttt{crit})} \label{subsec:stopping}

The $\texttt{SMC-KQ-KL}$ algorithm is designed to track the RMSE as $t$ is increased.
However, the RMSE is not available in closed form.
In this section we derive a tight upper bound on the RMSE that is used for the $\texttt{crit}$ component in Alg. \ref{SMCKQ}.

From the worst-case characterisation of KQ presented in Sec. \ref{subsec:overview}, we have an upper bound
\begin{eqnarray}
|\hat{\Pi}(f) - \Pi(f)| \leq e_n(\bm{w}; \{\bm{x}_j\}_{j=1}^n) \|f\|_{\mathcal{H}}. \label{CS eqn}
\end{eqnarray}
The term $e_n(\bm{w}; \{\bm{x}_j\}_{j=1}^n)$, denoted henceforth as $e_n(\{\bm{x}_j\}_{j=1}^n)$ (since $\bm{w}$ depends on $\{\bm{x}_j\}_{j=1}^n$), can be computed in closed form (see the Appendix).
This motivates the following upper bound on MSE:
\begin{eqnarray}
\mathbb{E} [\hat{\Pi}(f) - \Pi(f) ]^2 \leq \underbrace{\mathbb{E}[e_n(\{\bm{x}_j\}_{j=1}^n)^2]}_{(*)} \underbrace{\|f\|_{\mathcal{H}}^2}_{(**)}
\label{eq:upper bound}
\end{eqnarray}
The term $(*)$ can be estimated with the bootstrap approximation
\begin{eqnarray*}
\mathbb{E}[e_n(\{\bm{x}_j\}_{j=1}^n)^2] = \sum_{m=1}^M \frac{e_n(\{\tilde{\bm{x}}_{m,j}\}_{j=1}^n)^2}{M} =: R^2
\end{eqnarray*}
where $\tilde{\bm{x}}_{m,j}$ are independent draws from $\{\bm{x}_j\}_{j=1}^N$.
In $\texttt{SMC-KQ}$ the term $(**)$ is an unknown constant and the statistic $R$, an empirical proxy for the RMSE, is monitored at each iteration.
The algorithm terminates once an increase in this statistic occurs.
For $\texttt{SMC-KQ-KL}$ the term $(**)$ is non-constant as it depends on the kernel hyper-parameters; then $(**)$ can in addition be estimated as $\|\hat{f}\|_{\mathcal{H}}^2 = \bm{w}^\top \bm{\mathrm{K}}_\theta \bm{w}$ and we monitor the product of $R$ and $\|\hat{f}\|_{\mathcal{H}}$, with termination when an increase is observed (c.f. \texttt{test}, defined in the Appendix).

Full pseudo-code for $\texttt{SMC-KQ}$ is provided as Alg. \ref{SMCKQ}, while $\texttt{SMC-KQ-KL}$ is Alg. \ref{SMCKQKL} in the Appendix.
To summarise, we have developed a novel procedure, $\texttt{SMC-KQ}$ (and an extension $\texttt{SMC-KQ-KL}$), designed to approximate the optimal KQ estimator based on the unavailable optimal distribution in Eqn. \ref{objective function} where $\mathcal{F}$ is the unit ball of $\mathcal{H}$.
Earlier empirical results in Sec. \ref{subsec:motivation} suggest that $\texttt{SMC-KQ}$ has potential to provide a powerful and general algorithm for numerical integration. The additional computational cost of optimising the sampling distribution does however have to be counterbalanced with the potential gain in error, and so this method will mainly be of practical interest for problems with expensive integrands or complex target distributions.
The following section reports experiments designed to test this claim.



\section{RESULTS} \label{sec:results}

Here we compared $\texttt{SMC-KQ}$ (and $\texttt{SMC-KQ-KL}$) against the corresponding default approaches $\texttt{KQ}$ (and $\texttt{KQ-KL}$) that are based on $\Pi' = \Pi$.
Sec. \ref{sec:application_toy_problem} below reports an assessment in which the true value of integrals is known by design, while in Sec. \ref{Stein sec} the methods were deployed to solve a parameter estimation problem involving differential equations.

\subsection{Simulation Study} \label{sec:application_toy_problem}

To continue our illustration from Sec. \ref{sec:background}, we investigated the performance of $\texttt{SMC-KQ}$ and $\texttt{SMC-KQ-KL}$ for integration of $f(x) = 1 + \sin(2 \pi x)$ against the distribution $\Pi = \mathrm{N}(0,1)$.
Here the reference distribution was taken to be $\Pi_0 = \mathrm{N}(0,8^2)$. 
All experiments employed SMC with $N = 300$ particles, random walk Metropolis transitions (Alg. \ref{markov}), the re-sample threshold $\rho = 0.95$ and a maximum grid size $\Delta = 0.1$. 
Dependence of the subsequent results on the choice of $\Pi_0$ was investigated in Fig. \ref{fig:comparison_startingdist} in the Appendix.

Fig. \ref{fig:res1} (top) reports results for $\texttt{SMC-KQ}$ against $\texttt{KQ}$, for fixed length-scale $\ell = 1$.
Corresponding results for $\texttt{SMC-KQ-KL}$ against $\texttt{KQ-KL}$ are shown in the bottom plot.
It was observed that $\texttt{SMC-KQ}$ (resp. $\texttt{SMC-KQ-KL}$) out-performed $\texttt{KQ}$ (resp. $\texttt{KQ-KL}$) in the sense that, on a per-function-evaluation basis, the MSE achieved by the proposed method was lower than for the standard method.
The largest reduction in MSE achieved was about 8 orders of magnitude (correspondingly 4 orders of magnitude in RMSE).
A fair approximation to the $\sigma = 2$ method, which is approximately optimal for $n = 75$ (c.f. results in Fig. \ref{fig:sample sensitive}), was observed. 
The termination criterion in Sec. \ref{subsec:stopping} was observed to be a good approximation to the optimal temperature $t^*$ (Fig. \ref{fig:res1_hist} in Appendix).
As an aside, we note that the MSE was gated at $10^{-16}$ for all methods due to numerical condition of the kernel matrix $\mathbf{K}$ (a known feature of the Gaussian kernel used in this experiment).

The investigation was extended to larger dimensions ($d=3$ and $d=10$) and more complex integrands $f$ in the Appendix. In all cases, considerable improvements were obtained using $\texttt{SMC-KQ}$ over $\texttt{KQ}$.

\begin{figure}[t!]
\includegraphics[width = \columnwidth]{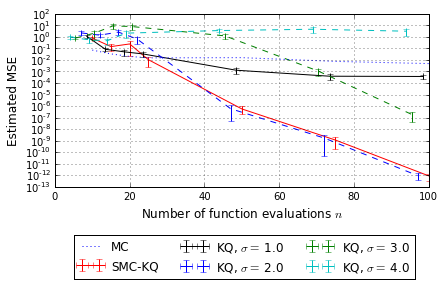}
\begin{center}
\includegraphics[width = \columnwidth]{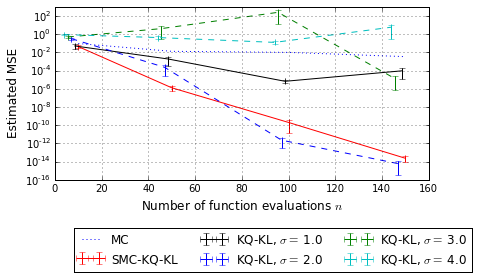}
\end{center}
\caption{Performance on for the running illustration of Figs. \ref{fig:sample sensitive} and \ref{fig:kernel sensitive}. The top plot shows $\texttt{SMC-KQ}$ against $\texttt{KQ}$, whilst the bottom plot illustrates the versions with kernel learning.}
\label{fig:res1}
\end{figure}

\subsection{Inference for Differential Equations} \label{Stein sec}

Consider the model given by
$\mathrm{d}x/\mathrm{d}t = f(t | \bm{\theta})$
with solution $x(t | \bm{\theta})$ depending on unknown parameters $\bm{\theta}$. Suppose we can obtain observations through the following noise model (likelihood): $y(t_i) = x(t_i|\bm{\theta}) + e_i$ at times $0 = t_1 < \ldots < t_n$ where we assume $e_i \sim N(0,\sigma^2)$ for known $\sigma>0$. Our goal is to estimate $x(T|\bm{\theta})$ for a fixed (potentially large) $T>0$. To do so, we will use a Bayesian approach and specify a prior $p(\bm{\theta})$, then obtain samples from the posterior $\pi(\bm{\theta}) := p(\bm{\theta} | y)$ using MCMC. The posterior predictive mean is then defined as: $\Pi\big(x(T|\bm{\cdot})\big)  =  \int x(T |\bm{\theta}) \pi(\bm{\theta}) \mathrm{d} \bm{\theta}$, and this can be estimated using an empirical average from the posterior samples. This type of integration problem is particularly challenging as the integrand requires simulating from the differential equation at each iteration. Furthermore, the larger $T$ or the smaller the grid, the longer the simulation will be and the higher the computational cost.

For a tractable test-bed, we considered Hooke's law, given by the following second order homogeneous ODE  given by 
\begin{equation*}
\theta_5 \frac{\mathrm{d}^2 x}{\mathrm{d} t^2}+\theta_4 \frac{\mathrm{d}x}{\mathrm{d}t}+\theta_3 x = 0,
\end{equation*}
with initial conditions $x(0)=\theta_1$ and $x'(0)=\theta_2$.  
This equation represents the evolution of a mass on a spring with friction \citep[Chapter 13]{Robinson2004}. More precisely, $\theta_3$ denotes the spring constant, $\theta_4$ the damping coefficient representing friction and $\theta_5$ the mass of the object. Since this differential equation is an overdetermined system we fixed $\theta_5=1$. In this case, if $\theta_4^2 \leq 4 \theta_3$, we get a damped oscillatory behaviour as presented in Fig. \ref{fig:ODE_1} (top). Data were generated with $\sigma=0.4$, $(\theta_1,\theta_2,\theta_3,\theta_4)=(1,3.75,2.5,0.5)$.          with log-normal priors with scale equal to $0.5$ for all parameters.

To implement \texttt{KQ} under an unknown normalisation constant for $\Pi$, we followed \citet{Oates2017} and made use of a Gaussian kernel that was adapted with Stein's method (see the Appendix for details). 
The reference distribution $\Pi_0$ was an wide uniform prior on the hypercube $[0,10]^4$. Brute force computation was used to obtain a benchmark value for the integral. For the SMC algorithm, an independent lognormal transition kernel was used at each iteration with parameters automatically tuned to the current set of particles.
Results in Fig. \ref{fig:ODE_1} demonstrate that \texttt{SMC-KQ} outperforms \texttt{KQ} for these integration problems.
These results improve upon those reported in \citet{Oates2016} for a similar integration problem based on parameter estimation for differential equations.

\begin{figure}[t!]
\vspace{2mm}
\begin{center}
\includegraphics[width = 0.7\columnwidth]{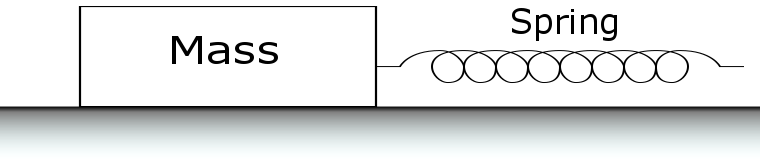}
\end{center}
\vspace{2mm}
\includegraphics[width = \columnwidth]{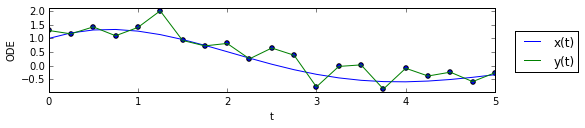}
\includegraphics[width = \columnwidth]{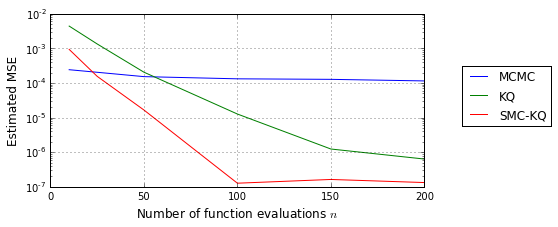}
\caption{Comparison of \texttt{SMC-KQ} and \texttt{KQ} on the ODE inverse problem. The top plot illustrates the physical system, the middle plot shows observations of the ODE, whilst the bottom plot illustrates the superior performance of \texttt{SMC-KQ} against \texttt{KQ}. }
\label{fig:ODE_1}
\end{figure}

\section{DISCUSSION} \label{sec:discussion}

In this paper we formalised the optimal sampling problem for KQ.
A general, practical solution was proposed, based on novel use of SMC methods.
Initial empirical results demonstrate performance gains relative to standard approach of KQ with $\Pi' = \Pi$.
A more challenging example based on parameter estimation for differential equations was used to illustrate the potential of \texttt{SMC-KQ} for Bayesian computation in combination with Stein's method.

Our methods were general but required user-specified choice of an initial distribution $\Pi_0$.
For compact state spaces $\mathcal{X}$ we recommend taking $\Pi_0$ to be uniform.
For non-compact spaces, however, there is a degree of flexibility here and default solutions, such as wide Gaussian distributions, necessarily require user input.
However, the choice of $\Pi_0$ is easier than the choice of $\Pi'$ itself, since $\Pi_0$ is not required to be optimal.
In our examples, improved performance (relative to standard KQ) was observed for a range of reference distributions $\Pi_0$. 

A main motivation for this research was to provide an alternative to optimisation-based KQ that alleviates strong dependence on the choice of kernel (Sec. \ref{no SBQ}).
This paper provides essential groundwork toward that goal, in developing sampling-based methods for KQ in the case of complex and expensive integration problems.
An empirical comparison of sampling-based and optimisation-based methods is reserved for future work.

Two extensions of this research are identified:
First, the curse of dimension that is intrinsic to standard Sobolev spaces can be alleviated by demanding `dominating mixed smoothness'; our methods are compatible with these (essentially tensor product) kernels \citep{Dick2013}.
Second, the use of sequential QMC \citep{Gerber2015} can be considered, motivated by further orders of magnitude reduction in numerical error observed for deterministic point sets (see Fig. \ref{fig:KQ-Halton} in the Appendix).

\subsubsection*{Acknowledgements}

FXB was supported by the EPSRC grant [EP/L016710/1]. CJO \& MG we supported by the Lloyds Register Foundation Programme on Data-Centric Engineering. WYC was supported by the ARC Centre of Excellence in Mathematical and Statistical Frontiers. MG was supported by the EPSRC grants [EP/J016934/3, EP/K034154/1, EP/P020720/1], an EPSRC Established Career Fellowship, the EU grant [EU/259348], a Royal Society Wolfson Research Merit Award. FXB, CJO, JC \& MG were also supported by the SAMSI working group on Probabilistic Numerics.

\subsubsection*{References}

\bibliographystyle{plainnat}
\begingroup
\renewcommand{\section}[2]{}%
\bibliography{SMCKQ_bib}
\endgroup

\clearpage

\newpage


\appendix
\setcounter{page}{1}

\section{Appendix}

This appendix complements the paper ``On the sampling problem for kernel quadrature". Section \ref{appendix:robustness_greedy_opt} discusses the potential lack of robustness of greedy optimization methods, which motivated the development of $\texttt{SMC-KQ}$. Sections \ref{appendix:additional_definitions} and \ref{appendix:theoretical_results} discuss some of the theoretical aspects of KQ, whilst Section \ref{appendix:R} and \ref{appendix:experimental_results} presents additional numerical experiments and details for implementation. Finally, Section \ref{appendix:algorithms_pseudocode} provides detailed pseudo-code for all algorithms used in this paper.

\subsection{Lack of Robustness of Optimisation Methods} \label{appendix:robustness_greedy_opt}

To demonstrate the non-robustness to mis-specified kernels, that is a feature of optimisation-based methods, we considered integration against $\Pi = \mathrm{N}(0,1)$ for functions that can be approximated by the kernel $k(x,x') = \exp(-(x-x')^2 / \ell^2)$.
An initial state $x_1$ was fixed at the origin and then for $n = 2,3,\dots$ the state $x_n$ was chosen to minimise the error criterion $e_n(\bm{w} ; \{x_j\}_{j=1}^n )$ given the location of the $\{x_j\}_{j = 1}^n$.
This is known as `sequential Bayesian quadrature' \citep[\texttt{SBQ};][]{Huszar2012,Gunter2014,Briol2015}.
The kernel length scale was fixed at $\ell = 0.01$ and we consider (as a thought experiment, since it does not enter into our selection of points) a more regular integrand, such as that shown in Fig. \ref{SBQ} (top).
The location of the states $\{x_j\}_{j=1}^n$ obtained in this manner are shown in Fig. \ref{SBQ} (bottom).
It is clear that \texttt{SBQ} is not an efficient use of computation for integration of the integrand against $\mathrm{N}(0,1)$. Of course, a bad choice of kernel length scale parameter $\ell$ can in principle be alleviated by kernel learning, but this will not be robust the case where $n$ is very small.

This example motivates sampling-based methods as an alternative to optimisation-based methods.
Future work will be required to better understand when methods such as \texttt{SBQ} can be reliable in the presence of unknown kernel parameters, but this was beyond the scope of this work.

\begin{figure}[t!]
\includegraphics[width = 0.5\textwidth]{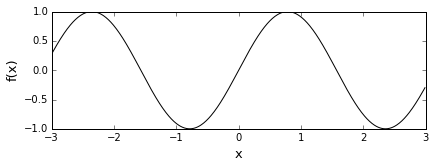}
\includegraphics[width = 0.5\textwidth]{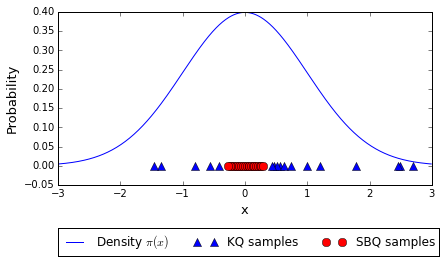}
\caption{Sequential minimisation of the error criterion $e_n(\bm{w}; \{\bm{x}_j\}_{j=1}^n )$, denoted \texttt{SBQ}, does not lead to adequate placement of points $\{\bm{x}_j\}_{j=1}^n$ when the kernel is mis-specified.
[Here the kernel length scale was fixed to $\ell = 0.01$.
Selected points $x_j$ are represented as red.
For comparison, a collection of draws from $\Pi$, as used in \texttt{KQ}, are shown as blue points.]}
\label{SBQ}
\end{figure}

\subsection{Additional Definitions}
\label{appendix:additional_definitions}

The space $L_2(\Pi)$ is defined to be the set of $\Pi$-measurable functions $f : \mathcal{X} \rightarrow \mathbb{R}$ such that the Lebesgue integral
$$
\int_{\mathcal{X}} f^2 \; \mathrm{d}\Pi
$$
exists and is finite.

For a multi-index $\bm{\alpha} = (\alpha_1,\dots,\alpha_d)$ define $|\bm{\alpha}| = \alpha_1 + \dots + \alpha_d$.
The (standard) Sobolev space of order $s \in \mathbb{N}$ is denoted 
\begin{eqnarray*}
\mathbb{H}_s(\Pi) \; = \; \{ f : \mathcal{X} \rightarrow \mathbb{R} \text{ s.t. } \hspace{110pt} \\
(\partial x_1)^{\alpha_1} \dots (\partial x_d)^{\alpha_d} f \in L_2(\Pi) \; \forall \; |\bm{\alpha}| \leq s \}.
\end{eqnarray*}
This space is equipped with norm
$$
\| f \|_{\mathbb{H}_s(\Pi)} = \Bigg( \sum_{|\bm{\alpha}| \leq s} \|(\partial x_1)^{\alpha_1} \dots (\partial x_d)^{\alpha_d} f\|_{L_2(\Pi)}^2 \Bigg)^{1/2}.
$$

Two normed spaces $(\mathcal{F},\|\cdot\|)$ and $(\mathcal{F},\|\cdot\|')$ are said to be `norm equivalent' if there exists $0 < c < \infty$ such that 
$$
c^{-1} \|f\|' \leq \|f\| \leq c \|f\|'
$$
for all $f \in \mathcal{F}$.

\subsection{Theoretical Results} \label{appendix:theoretical_results}

\subsubsection{Proof of Theorem \ref{theo:BMC}}

\begin{proof}
From Thm. 11.13 in \citet{Wendland2004} we have that there exist constants $0 < c_k < \infty$, $h_0 > 0$ such that
\begin{eqnarray}
|\hat{f}(\bm{x}) - f(\bm{x})| \leq c_k h_n^s \|f\|_{\mathcal{H}} \label{linf bound}
\end{eqnarray}
for all $\bm{x} \in \mathcal{X}$, provided $h_n < h_0$, where
$$
h_n = \sup_{\bm{x} \in \mathcal{X}} \min_{i = 1,\dots,n} \|\bm{x} - \bm{x}_i \|_2.
$$
Under the hypotheses, we can suppose that the deterministic states $\bm{x}_1,\dots,\bm{x}_m$ ensure $h_m < h_0$.
Then Eqn. \ref{linf bound} holds for all $n > m$, where the $\bm{x}_{m+1}, \dots, \bm{x}_n$ are independent draws from $\Pi'$.
It follows that
\begin{eqnarray*}
|\hat{\Pi}(f) - \Pi(f)| & \leq & \sup_{\bm{x} \in \mathcal{X}} |\hat{f}(\bm{x}) - f(\bm{x})| \\
& \leq & c_k h_n^s \|f\|_{\mathcal{H}}.
\end{eqnarray*}

Next, Lem. 1 in \citet{Oates2016} establishes that, under the present hypotheses on $\mathcal{X}$ and $\Pi'$, there exists $0 < c_{\Pi',\epsilon} < \infty$ such that
$$
\mathbb{E} [h_n^{2s}] \leq c_{\Pi',\epsilon} m^{-2s/d + \epsilon}
$$
for all $\epsilon > 0$, where $c_{\Pi',\epsilon}$ is independent of $n$.

Combining the above results produces
\begin{eqnarray*}
\mathbb{E} [\hat{\Pi}(f) - \Pi(f) ]^2 & \leq & c_k^2 \mathbb{E}[h_n^{2s}] \|f\|_{\mathcal{H}}^2 \\
& \leq & c_k^2 c_{\Pi',\epsilon} m^{-2s/d + \epsilon} \|f\|_{\mathcal{H}}^2
\end{eqnarray*}
as required, with $c_{k,\Pi',\epsilon} = c_k c_{\Pi',\epsilon}^{1/2}$.
\end{proof}

\subsubsection{Proof of Theorem \ref{bach result}}

\begin{proof}
The Cauchy-Schwarz result for kernel mean embeddings \citep{Smola2007} gives
\begin{eqnarray}
& & |\hat{\Pi}(f) - \Pi(f)| \label{CS} \\
& \leq & \left\| \sum_{i=1}^n w_i k(\cdot,\bm{x}_i) - \int_{\mathcal{X}} k(\cdot,\bm{x}) \Pi(\mathrm{d}\bm{x}) \right\|_{\mathcal{H}} \|f\|_{\mathcal{H}}. \nonumber
\end{eqnarray}
Consider the first term above.
Since $\mathcal{H}$ is dense in $L_2(\Pi)$, it follows that $\Sigma^{1/2}$ (the unique positive self-adjoint square root of $\Sigma$) is an isometry from $L_2(\Pi)$ to $\mathcal{H}$.
Now, since $k(\cdot,\bm{x}) \in \mathcal{H}$, there exists a unique element $\psi(\cdot,\bm{x}) \in L_2(\Pi)$ such that $\Sigma^{1/2} \psi(\cdot,\bm{x}) = k(\cdot,\bm{x})$.
Then we have that
\begin{eqnarray*}
& & \left\| \sum_{i=1}^n w_i k(\cdot,\bm{x}_i) - \int_{\mathcal{X}} k(\cdot,\bm{x}) \Pi(\mathrm{d}\bm{x}) \right\|_{\mathcal{H}} \\
& = & \left\| \sum_{i=1}^n w_i \Sigma^{1/2}\psi(\cdot,\bm{x}_i) - \int_{\mathcal{X}} \Sigma^{1/2}\psi(\cdot,\bm{x}) \Pi(\mathrm{d}\bm{x}) \right\|_{\mathcal{H}} \\
& = & \left\| \sum_{i=1}^n w_i \psi(\cdot,\bm{x}_i) - \int_{\mathcal{X}} \psi(\cdot,\bm{x}) \Pi(\mathrm{d}\bm{x}) \right\|_{L_2(\Pi)} .
\end{eqnarray*}

For $f \in L_2(\Pi)$, we have $f \in \mathcal{H}$ if and only if 
\begin{eqnarray}
f = \int_{\mathcal{X}} g(\bm{x}) \psi(\cdot,\bm{x}) \Pi(\mathrm{d}\bm{x}) \label{particular f}
\end{eqnarray}
for some $g \in L_2(\Pi)$, in which case $\|f\|_{\mathcal{H}}$ is equal to the infimum of $\|g\|_{L_2(\Pi)}$ under all such representations $g$.
In particular, it follows that $\|f\|_{\mathcal{H}} = 1$ for the particular choice with $g(\bm{x}) = 1$ for all $\bm{x} \in \mathcal{X}$.

Under the hypothesis on $n$, Prop. 1 of \citet{Bach2015} established that when $\bm{x}_1,\dots,\bm{x}_n \sim \Pi_{\text{B}}$ are independent, then
\begin{eqnarray*}
\sup_{\|f\|_{\mathcal{H}} \leq 1} \inf_{\|\bm{\beta}\|_2^2 \leq \frac{4}{n}} \left\| \sum_{i=1}^n \frac{\beta_i}{\pi_{\text{B}}(\bm{x}_i)^{1/2}} \psi(\cdot,\bm{x}_i) - f \right\|_{L_2(\Pi)}^2 \leq 4 \lambda
\end{eqnarray*}
with probability at least $1 - \delta$.
Fixing the function $f$ in Eqn. \ref{particular f} leads to the statement that
\begin{eqnarray*}
\inf_{\|\bm{\beta}\|_2^2 \leq \frac{4}{n}} \left\| \sum_{i=1}^n \frac{\beta_i}{\pi_{\text{B}}(\bm{x}_i)^{1/2}} \psi(\cdot,\bm{x}_i) - \int_{\mathcal{X}} \psi(\cdot,\bm{x}) \Pi(\mathrm{d}\bm{x}) \right\|_{L_2(\Pi)}^2 
\end{eqnarray*}
is at most $4 \lambda$ with probability at least $1 - \delta$.
The infimum over $\|\bm{\beta}\|_2^2 \leq 4/n$ can be replaced with an unconstrained infimum over $\mathbb{R}^n$ to obtain the weaker statement that
\begin{eqnarray*}
\inf_{\bm{\beta} \in \mathbb{R}^n} \left\| \sum_{i=1}^n \frac{\beta_i}{\pi_{\text{B}}(\bm{x}_i)^{1/2}} \psi(\cdot,\bm{x}_i) - \int_{\mathcal{X}} \psi(\cdot,\bm{x}) \Pi(\mathrm{d}\bm{x}) \right\|_{L_2(\Pi)}^2 
\end{eqnarray*}
is at most $4 \lambda$ with probability at least $1 - \delta$.
Now, recall from Sec. \ref{subsec:overview} that the KQ weights $\bm{w}$ are characterised through the solution $\bm{\beta}^*$ to this optimisation problem as $w_i = \beta_i^* \pi_{\text{B}}(\bm{x}_i)^{-1/2}$.
It follows that
\begin{eqnarray*}
\left\| \sum_{i=1}^n w_i \psi(\cdot,\bm{x}_i) - \int_{\mathcal{X}} \psi(\cdot,\bm{x}) \Pi(\mathrm{d}\bm{x}) \right\|_{L_2(\Pi)}^2 \leq 4 \lambda
\end{eqnarray*}
with probability at least $1 - \delta$.
Combining this fact with Eqn. \ref{CS} completes the proof.
\end{proof}

\subsubsection{$\Pi_{\text{B}}$ for the Example of Figure \ref{fig:sample sensitive}}

In this section we consider scope to derive $\Pi_{\text{B}}$ in closed-form for the example of Fig. \ref{fig:sample sensitive}.
The following will be used:

\begin{proposition}[Prop. 1 in \citet{Shi2009}] \label{Shi prop}
Let $\mathcal{X} = \mathbb{R}$, $\Pi = \mathrm{N}(\mu,\sigma^2)$ and $k(x,x') = \exp(-(x-x')^2/\ell^2)$.
Define $\beta = 4\sigma^2/\ell^2$ and denote the $j$th Hermite polynomial as $H_j(x)$. 
Then the eigenvalues $\mu_j$ and corresponding eigenfunctions $e_j$ of the integral operator $\Sigma$ are
\begin{eqnarray*}
\mu_j = \sqrt{\frac{2}{(1 + \beta + \sqrt{1+2 \beta})}} \times \Big( \frac{\beta}{1 + \beta + \sqrt{1+2 \beta}} \Big)^j
\end{eqnarray*}
and
\begin{eqnarray*}
e_j(x) = \frac{(1 + 2 \beta)^{1/8}}{\sqrt{2^j j!}} \exp\Big( - \frac{(x - \mu)^2}{2 \sigma^2} \frac{\sqrt{1+2\beta} -1}{2}\Big) \\
\times H_j\Big(\Big(\frac{1}{4}+\frac{\beta}{2}\Big)^{1/4} \frac{x - \mu}{\sigma}\Big)
\end{eqnarray*}  
for $j \in \{0,1,2,\dots\}$.
\end{proposition}

\begin{proposition}[Ex. 6.8 in \citet{Temme1996}, p.167] \label{generating lemma}
The bilinear generating function for Hermite polynomials is
\begin{eqnarray*}
\sum_{j=0}^\infty \frac{t^j}{j!} H_j(x) H_j(z) \hspace{110pt} \\
= \frac{1}{\sqrt{1 - 4t^2}} \exp\left( x^2 - \frac{(x - 2zt)^2}{1 - 4t^2} \right).
\end{eqnarray*}
\end{proposition}

\begin{proposition} \label{Bach exp deriv}
For the example in Fig. \ref{fig:sample sensitive} we have
\begin{eqnarray*}
\pi_{\text{B}}(x ; \lambda) \; \propto \; \hspace{155pt} \\
\hspace{10pt} \exp(-x^2) \sum_{j = 0}^\infty \frac{1}{1 + \lambda 2^{j+1}} \frac{1}{2^j j!} H_j^2 \Big(\sqrt{\frac{3}{2}} x \Big)  .
\end{eqnarray*}
\end{proposition} 
\begin{proof}
For the example of Fig. \ref{fig:sample sensitive}, in the notation of Prop. \ref{Shi prop}, we have $\mu = 0$, $\sigma = 1$, $\ell = 1$ and $\beta = 4$.
Thus 
\begin{eqnarray*}
\mu_j & = & \Big(\frac{1}{2}\Big)^{j+1}  \\
e_j(x)^2 & = & \sqrt{3} \exp(-x^2) \frac{1}{2^j j!} H_j^2 \Big(\sqrt{\frac{3}{2}} x \Big) 
\end{eqnarray*}
and so
\begin{eqnarray*}
\pi_{\text{B}}(x ; \lambda) \; \propto \; \sum_{j} \frac{\mu_j}{\mu_j + \lambda} e_j^2(x) \hspace{100pt} \\
\hspace{10pt} \propto \; \exp(-x^2) \sum_{j = 0}^\infty \frac{1}{1 + \lambda 2^{j+1}} \frac{1}{2^j j!} H_j^2 \Big(\sqrt{\frac{3}{2}} x \Big) 
\end{eqnarray*}
as required.
\end{proof}

To the best of our knowledge, the expression for $\Pi_{\text{B}}$ in Prop. \ref{Bach exp deriv} does not admit a closed form.
This poses a practical challenge.
However, some limited insight is available through basic approximations:
\begin{itemize}
\item For large values of $\lambda$ we have $1 + \lambda 2^{j+1} \approx \lambda 2^{j+1}$ for all $j \in \{0,1,2,\dots\}$, from which we obtain
\begin{eqnarray*}
\pi_{\text{B}}(x ; \lambda) & \appropto & \exp(-x^2) \sum_{j = 0}^\infty \frac{1}{4^j j!} H_j^2 \Big(\sqrt{\frac{3}{2}} x \Big) \\
& \propto & \exp(-x^2) \exp(x^2) \quad = \quad 1,
\end{eqnarray*}
where the second step made use of Prop. \ref{generating lemma}.
Thus when large integration errors are tolerated, $\Pi_{\text{B}}$ requires that we take the states $\bm{x}_i$ to be approximately uniform over $\mathcal{X}$ (of course, this limiting distribution is improper and serves only for illustration).
\item For small values of $\lambda$, the series in Prop. \ref{Bach exp deriv} is dominated by the first $m$ terms such that $j < m$ if and only if $\lambda 2^{j+1} < 1$.
Indeed, for $j \leq m$ we have $1 + \lambda 2^{j+1} \approx 1$.
Thus we have a computable approximation
\begin{eqnarray*}
\pi_{\text{B}}(x ; \lambda) & \appropto & \exp(-x^2) \sum_{j = 0}^m \frac{1}{2^j j!} H_j^2 \Big(\sqrt{\frac{3}{2}} x \Big) 
\end{eqnarray*}
where $m = \lceil - \log_2(\lambda) \rceil$.
Empirical results (not shown) indicate that this is not a useful approximation from a practical standpoint, since at finite $m$ the tails of the approximation are explosive (due to the use of a polynomial basis).
\end{itemize}
The approximation method in \citet{Bach2015} was also used to obtain the numerical approximation to $\Pi_{\text{B}}$ shown in Fig. \ref{fig:importance of regularisation}.
This appears to support the intuition that it is beneficial to over-sample from the tails of $\Pi$.

To finish, we remark that Prop. \ref{Bach exp deriv} implies that the integration error in this example scales as
$$
\sqrt{\mu_n} \sim 2^{-n/2}
$$
as $n \rightarrow \infty$ when samples are drawn from $\Pi_{\text{B}}$.
This agrees with both intuition and empirical results that concern approximation with exponentiated quadratic kernels.

\begin{figure}[t!]
\begin{center}
\includegraphics[width = 0.9\columnwidth]{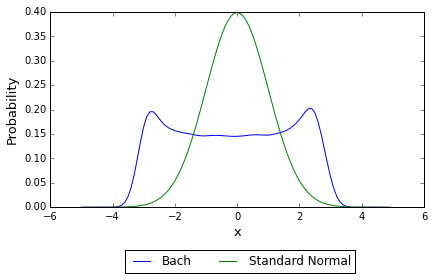}
\end{center}
\caption{Numerical approximation of $\Pi_{\text{B}}$ for the running illustration. Here the regularisation parameter was $\lambda=10^{-15}$.}
\label{fig:importance of regularisation}
\end{figure}

\subsubsection{Additional Theoretical Material}

As mentioned in the Main Text, the worst-case error $e_n(\{\bm{x}_j\}_{j=1}^n)$ can be computed in closed form:
$$
e_n(\{\bm{x}_j\}_{j=1}^n)^2 = \Pi \otimes \Pi(k) - 2 \bm{w}^\top \mathbf{K} \bm{z} + \bm{w}^\top \mathbf{K} \bm{w}
$$
Here we have defined
$$
\Pi \otimes \Pi(k) = \iint_{\mathcal{X} \times \mathcal{X}} k(\bm{x},\bm{x}') \; \Pi \otimes \Pi (\mathrm{d}\bm{x} \times \mathrm{d}\bm{x}') 
$$
where $\Pi \otimes \Pi$ is the product measure of $\Pi$ with itself.

Next, we report a result which does not address KQ itself, but considers importance sampling methods for integration of functions in a Hilbert space.
The following is due to \citet{Plaskota2009,Hinrichs2010} and we provide an elementary proof of their result:

\begin{theorem}
The assumptions of Sec. \ref{subsec:estres} are taken to hold. 
In addition, we assume that distributions $\Pi,\Pi'$ admit densities $\pi,\pi'$.
Introduce importance sampling estimators of the form
$$
\hat{\Pi}_{\mathrm{IS}}(f) = \frac{1}{n} \sum_{i=1}^n f(\bm{x}_i) \frac{\pi(\bm{x}_i)}{\pi'(\bm{x}_i)} ,
$$
where $\bm{x}_1,\dots,\bm{x}_n \sim \Pi'$ are independent, and consider the distribution $\Pi'$ that minimises 
$$
\sup_{f \in \mathcal{F}} \sqrt{\mathbb{E}[ \hat{\Pi}_{\mathrm{IS}}(f) - \Pi(f) ]^2}.
$$
For $\mathcal{F} = \{f\}$ we have that $\Pi'$ is $\pi'(\bm{x}) \propto |f(\bm{x})| \pi(\bm{x})$, while for $\mathcal{F} = \{f \in \mathcal{H} \; : \; \|f\|_{\mathcal{H}} \leq 1 \}$ we have that $\Pi'$ is $\pi'(\bm{x}) \propto \sqrt{k(\bm{x},\bm{x})}\pi(\bm{x})$.
\end{theorem}
\begin{proof}
The first result, for $\mathcal{F} = \{f\}$ is well-known; e.g. Thm. 3.3.4 in \citet{Robert2013}.

For the second case, where $\mathcal{F}$ is the unit ball in $\mathcal{H}$, we start by establishing a (tight) upper bound for the supremum of $f^2$ over $f \in \mathcal{F}$:
\begin{eqnarray*}
|f(\bm{x})| 
& = &
\big| \langle f, k(\cdot,\bm{x}) \rangle_{\mathcal{H}} \big| \\
& \leq &
\|f\|_\mathcal{H} \|k(\cdot,\bm{x})\|_\mathcal{H} \\
& = &
\|f\|_\mathcal{H} \sqrt{\langle k(\cdot,\bm{x}), k(\cdot,\bm{x}) \rangle_{\mathcal{H}}} \\
& = &
\|f\|_\mathcal{H} \sqrt{k(\bm{x},\bm{x})}
\end{eqnarray*}
where the inequality here is Cauchy-Schwarz. 
Squaring both sides and taking the supremum over $f \in \mathcal{F}$ gives
\begin{equation}
\sup_{f \in \mathcal{F}} f(\bm{x})^2
\; \leq \; 
\sup_{f \in \mathcal{F}}\|f\|_\mathcal{H}^2 \;
k(\bm{x},\bm{x})
\; = \;
k(\bm{x},\bm{x}). \label{rkhs bound}
\end{equation}
This is in fact an equality, since for given $\bm{x} \in \mathcal{X}$ we can take $f(\bm{x}') = k(\bm{x}',\bm{x}) / \sqrt{k(\bm{x},\bm{x})}$ which has $\|f\|_{\mathcal{H}} = 1$ and $f(\bm{x})^2 = k(\bm{x},\bm{x})$.

Our objective is expressed as
\begin{eqnarray*}
\sup_{f \in \mathcal{F}} \sqrt{\mathbb{E}[ \hat{\Pi}_{\text{IS}}(f) - \Pi(f) ]^2} = \sup_{f \in \mathcal{F}} \frac{1}{\sqrt{n}} \; \text{Std}\Big(\frac{f\pi}{\pi'} ; \Pi'\Big) 
\end{eqnarray*}
and since
$$
\text{Std}\Big(\frac{f\pi}{\pi'} ; \Pi'\Big)^2 \; = \; \Pi'\Big(\Big(\frac{f\pi}{\pi'}\Big)^2\Big) - \Pi'\Big(\frac{f\pi}{\pi'}\Big)^2
$$
we thus aim to minimise
\begin{equation*} 
\sup_{f \in \mathcal{F}} \; \Pi'\Big(\Big(\frac{f\pi}{\pi'}\Big)^2\Big) \label{fx objective}
\end{equation*}
over $\Pi' \in \mathcal{P}(\mathcal{F} \cdot \mathrm{d}\Pi / \mathrm{d}\Pi')$.
(Here $\mathcal{F} \cdot \mathrm{d}\Pi / \mathrm{d}\Pi'$ denotes the set of functions of the form $f \cdot \mathrm{d}\Pi / \mathrm{d}\Pi'$ such that $f \in \mathcal{F}$.)

Combining Eqns. \ref{rkhs bound} and \ref{fx objective}, we have
\begin{eqnarray*}
\sup_{f \in \mathcal{F}} \Pi'\Big(\Big(\frac{f\pi}{\pi'}\Big)^2\Big) 
& \leq & \Pi'\Big( \sup_{f \in \mathcal{F}} \Big(\frac{f \pi}{\pi'}\Big)^2\Big) \\
& = & \Pi'\Big( k(\cdot,\cdot)\Big(\frac{\pi(\cdot)}{\pi'(\cdot)}\Big)^2\Big)
\end{eqnarray*}
As before, this is in fact an equality, as can be seen from $f(\bm{x}) = \sqrt{k(\bm{x},\bm{x})}$.

From Jensen's inequality,
\begin{eqnarray}
\Pi'\Big( k(\cdot,\cdot)\Big(\frac{\pi(\cdot)}{\pi'(\cdot)}\Big)^2\Big)
& \geq & \Big(\Pi'\Big(\sqrt{k(\cdot,\cdot)} \frac{\pi(\cdot)}{\pi'(\cdot)}\Big)\Big)^2 \label{Jensen} \\
& = & \big(\Pi\big(\sqrt{k(\cdot,\cdot)}\big)\big)^2. \nonumber
\end{eqnarray}
Since the right hand side is independent of $\Pi'$, a choice of $\Pi'$ for which Eqn. \ref{Jensen} is an equality must be a minimiser of Eqn. \ref{fx objective}.
It remains just to verify this fact for $\pi'(\bm{x}) =  \sqrt{k(\bm{x},\bm{x})} \pi(\bm{x}) / C$, where the normalising constant is $C=\Pi(\sqrt{k(\cdot,\cdot)})$.
For this choice
\begin{eqnarray*}
\Pi' \Big( k(\cdot,\cdot) \Big( \frac{\pi(\cdot)}{\pi'(\cdot)} \Big)^2 \Big)
& = &
 \Pi'(C^2) \\
& = &
(\Pi(\sqrt{k(\cdot,\cdot)}))^2 
\end{eqnarray*} 
as required.
\end{proof}

\FloatBarrier

\subsection{Implementation of $\texttt{test}(R < R_{\min})$}
\label{appendix:R}
Here we provide details for how the criterion $R < R_{\min}$ was tested.
The problem with the naive approach of comparing $R$ estimated at $t_{i-1}$ directly with $R$ estimated at $t_i$ is that Monte Carlo error can lead to an incorrect impression that $R$ is increasing, when it is in fact decreasing, and cause the algorithm to terminate when estimation is poor (see Fig. \ref{fig:smoother} and note the jaggedness of the estimated $R$ curve as a function of inverse temperature $t$).
Our solution was to apply a least-squares linear smoother to the estimates for $R$ over 5 consecutive temperatures.
This approach, denoted \texttt{test}, illustrated in Fig. \ref{fig:smoother}, determines whether the gradient of the linear smoother is positive or negative, and in this way we are able to provide robustness to Monte Carlo error in the termination criterion.
To be precise, the algorithm requires at least 5 temperature evaluations before termination is considered (Fig. \ref{fig:smoother}; left) and terminates when the gradient of the linear smoother becomes positive for the first time (Fig. \ref{fig:smoother}; right).
The success of this strategy was established in Fig. \ref{fig:res1_hist} later in the Appendix.

\begin{figure*}[t!]
\centering
\includegraphics[width = 0.32\textwidth]{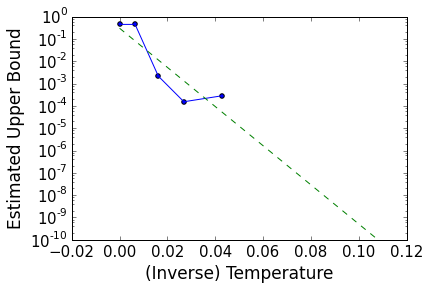}
\includegraphics[width = 0.32\textwidth]{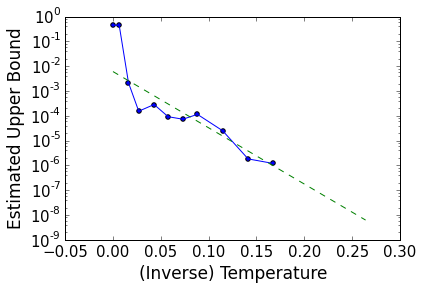}
\includegraphics[width = 0.32\textwidth]{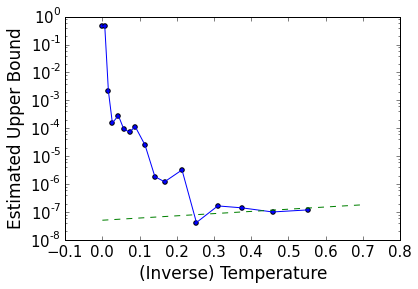}
\caption{Implementation of $\texttt{test}(R < R_{\min})$. A linear smoother (dashed line) was based on 5 consecutive (inverse) temperature parameters $t_{i-4},t_{i-3},t_{i-2},t_{i-1},t_i$.
To begin it is required that 5 temperatures are considered (left panel).
The algorithm terminates on the first occasion when the linear smoother takes a positive gradient (right panel).}
\label{fig:smoother}
\end{figure*}

\subsection{Experimental Results}
\label{appendix:experimental_results}

\subsubsection{Implementation of Simulation Study}

Denote by $\mathrm{N}(\bm{x}|\bm{\mu},\bm{\Sigma})$ the p.d.f. of the multivariate Gaussian distribution with mean $\bm{\mu}$ and covariance $\bm{\Sigma}$. Furthermore, we denote by $\bm{\Sigma}_{\sigma}$ the diagonal covariance matrix with diagonal element $\sigma^2$.
Then elementary manipulation of Gaussian densities produces:
\begin{eqnarray*}
k(\bm{x},\bm{y}) & := & \exp\Big(-\frac{\sum_{j=1}^d \big(x_j-y_j\big)^2}{l^2}\Big) \\
& = & (\sqrt{\pi} l)^d \phi\big(\bm{x}|\bm{y},\Sigma_{l/\sqrt{2}}\big) \\
\nabla_l k(x,y) &:=& \frac{2\sum_{j=1}^d (x_j-y_j)^2}{l^3}k(\bm{x},\bm{y})  \\ 
\Pi[k(\cdot,\bm{x})] & := &(\sqrt{\pi} l)^d \mathrm{N}\big(\bm{x}|\bm{0},\Sigma_{\sigma}+\Sigma_{l/\sqrt{2}}\big) \\
\Pi \otimes \Pi(k) & := &(\sqrt{\pi} l)^d \mathrm{N}\big(\bm{0}|\bm{0},\Sigma_{\sqrt{2}\sigma}+\Sigma_{l/\sqrt{2}}\big)
\end{eqnarray*}

\subsubsection{Dependence on Parameters for the Simulation Study}

For the running illustration with $f(x) = 1 + \sin(x)$, $\Pi = \text{N}(0,1)$, $\Pi' = \text{N}(0,\sigma^2)$ and $k(x,x') = \exp(-(x-x')^2 / \ell^2)$,
we explored how the RMSE of KQ depends on the choice of both $\sigma$ and $\ell$.
Here we go beyond the results presented in Fig. \ref{fig:kernel sensitive}, which considered fixed $n$, to now consider the simultaneous choice of both $\sigma, \ell$ for varying $n$.
Note that in these numerical experiments the kernel matrix inverse $\mathbf{K}^{-1}$ was replaced with the regularised inverse $(\mathbf{K} + \lambda \mathbf{I})^{-1}$ that introduces a small `nugget' term $\lambda > 0$ for stabilisation.
Results, shown in Fig. \ref{fig2 ctd}, demonstrate two principles that guided the methodological development in this paper:
\begin{itemize}
\item Length scales $\ell$ that are `too small' to learn from $n$ samples do not permit good approximations $\hat{f}$ and lead in practice to high RMSE. 
At the same time, if $\ell$ is taken to be `too large' then efficient approximation at size $n$ will also be sacrificed.
This is of course well understood from a theoretical perspective and is borne out in our empirical results.
These results motivated extension of \texttt{SMC-KQ} to \texttt{SMC-KQ-KL}.

\item In general the `sweet spot', where $\sigma$ and $\ell$ lead to minimal RMSE, is quite small. However, the problem of optimal choice for $\sigma$ and $\ell$ does not seem to become more or less difficult as $n$ increases. 
This suggests that a method for selection of $\sigma$ (and possibly also of $\ell$) ought to be effective regardless of the number $n$ of states that will be used.
\end{itemize}

\begin{figure}[t!]
\begin{center}
\includegraphics[width = \columnwidth]{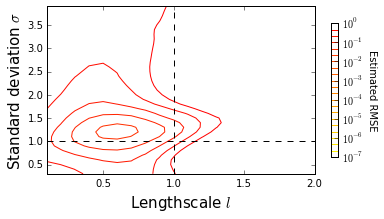}
\end{center}
\begin{center}
\includegraphics[width = \columnwidth]{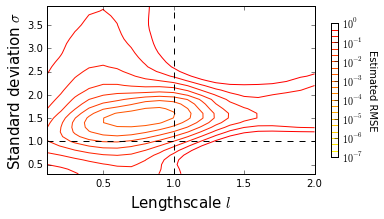}
\end{center}
\begin{center}
\includegraphics[width = \columnwidth]{figures/l_vs_sig_RMSE_n75_rep300.png}
\end{center}
\caption{Example of Fig. \ref{fig:kernel sensitive}, continued.
Here we consider the simultaneous choice of sampling standard deviation $\sigma$ and kernel length-scale $\ell$, reporting empirical estimates for the estimated root mean square integration error (over $M=300$ repetitions) in each case for sample size (a) $n = 25$ (top), (b) $n = 50$ (middle) and (c) $n = 75$ (bottom).}
\label{fig2 ctd}
\end{figure}

\subsubsection{Additional Results for the Simulation Study}

To understand whether the termination criterion of Sec. \ref{subsec:stopping} was suitable (and, by extension, to examine the validity of the convexity ansatz in Sec. \ref{subsec:SMC}), in Fig. \ref{fig:res1_hist} we presented histograms for both estimated and actual optimal (inverse) temperature parameter $t^*$.
Results supported the use of the criterion, in the form described above for $\texttt{test}$.

In Fig. \ref{fig:comparison_startingdist} reports the dependence of performance on the choice of initial distribution $\Pi_0$. 
There was relatively little influence on the RMSE obtained by the method for this wide range of initial distribution, which supports the purported robustness of the method. 

We also test the method on more complex integrands in Fig. \ref{fig:other_f}: $f(x)=1 + \sin(4\pi x)$ and $f(x)=1 + \sin(8\pi x)$. These are more challenging for KQ compared to the illustration in the Main Text, since they are more difficult to interpolate due to their higher periodicity. However, $\texttt{SMC-KQ}$ still manages to adapt to the complexity of the integrand and performs as well as the best importance sampling distribution ($\sigma=2$). 

As an extension, we also study the robustness to the dimensionality to the problem. In problem, we consider the generalisation of our main test function to $f:\mathbb{R}^d \rightarrow \mathbb{R}$ given by $f(\bm{x})=1 + \prod_{j=1}^d \sin(2\pi x_j)$. Notice that the integral can still be computed analytically and equals $1$. We present results for $d=2$ and $d=3$ in Fig. \ref{fig:larger_d}. 
These two cases are more challenging for both the $\texttt{KQ}$ and $\texttt{SMC-KQ}$ methods, since the higher dimension implies a slower convergence rate. Once again, we notice that $\texttt{SMC-KQ}$ manages to adapt to the complexity of the problem at hand, and provides improved performance on simpler sampling distributions.

Finally, we considered replacing the independent samples $x_j \sim \Pi$ with samples drawn from a quasi-random point sequence.
Fig. \ref{fig:KQ-Halton} reports results where draws from $\mathrm{N}(0,1)$ were produced based on a Halton quasi-random number generator. In this case, the performance is improved by up to 10 orders of magnitude in MSE when the sampling is done with respect to a range of tempered sampling distribution (here $\mathrm{N}(0,3^2)$). 
This suggests that a SQMC approach \citep{Gerber2015} could provide further improvement and this suggested for future work.

\begin{figure}[t!]
\includegraphics[width = 0.49\columnwidth]{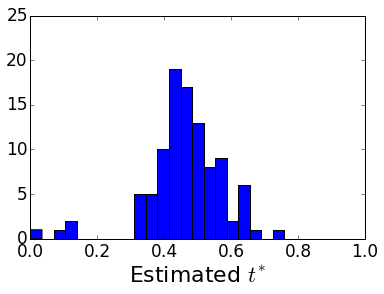}
\includegraphics[width = 0.49\columnwidth]{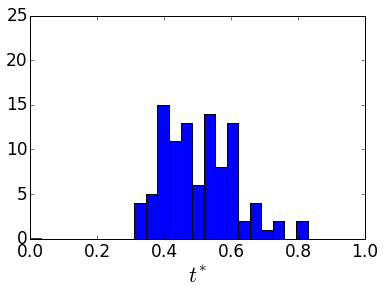}
\caption{Histograms for the optimal (inverse) temperature parameter $t^*$.
Left: Estimate of $t^*$ provided under the termination criterion of Sec. \ref{subsec:stopping}.
Right: Estimate of $t^*$ obtained by estimating $R$ over a grid for $t \in [0,1]$ and returning the global minimum.
The similarity of these histograms is supportive of the convexity ansatz in Sec. \ref{subsec:SMC}.}
\label{fig:res1_hist}
\end{figure}

\begin{figure}[t!]
\begin{center}
\includegraphics[width = \columnwidth]{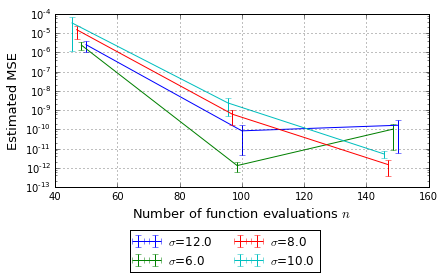}
\end{center}
\caption{Comparison of the performance of \texttt{SMC-KQ} on the running illustration of Figs. \ref{fig:sample sensitive} and \ref{fig:kernel sensitive} for varying initial distribution $\Pi_0 = \mathrm{N}(0,\sigma^2)$. }
\label{fig:comparison_startingdist}
\end{figure}

\begin{figure}[t!]
\includegraphics[width = \columnwidth]{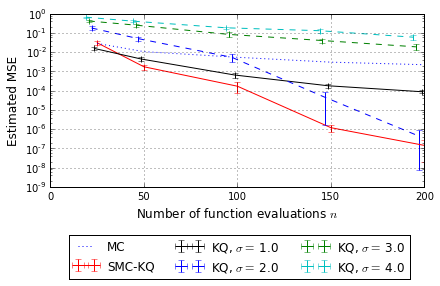}
\includegraphics[width = \columnwidth]{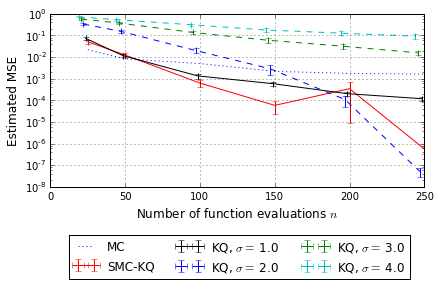}
\caption{Performance of \texttt{KQ} and \texttt{SMC-KQ} on the integration problem with $f(x)=1+\sin(4\pi x)$ (top) and $f(x)=1+\sin(8\pi x)$ (bottom) integrated against $ \mathrm{N}(0,1)$. The SMC sampler was initiated with a $\mathrm{N}(0,8^2)$ distribution. 
The kernel used was Gaussian with length scales $\ell=0.25$ (top) and $\ell=0.15$ (bottom) each chosen to reflect the complexity of the functions.}
\label{fig:other_f}
\end{figure}

\begin{figure}[t!]
\includegraphics[width = \columnwidth]{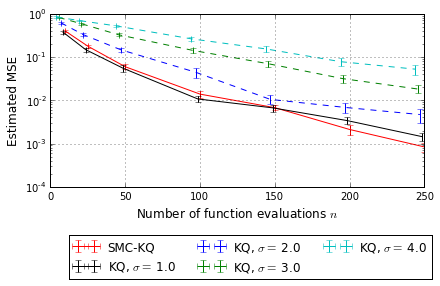}
\vspace{3mm}
\includegraphics[width = \columnwidth]{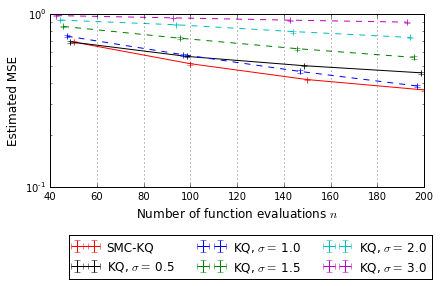}
\vspace{3mm}
\includegraphics[width = \columnwidth]{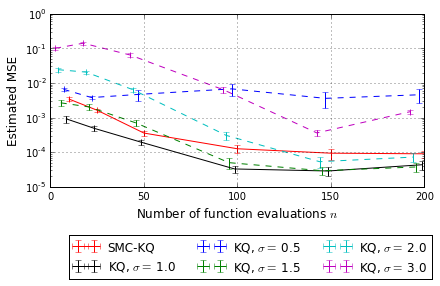}
\vspace{3mm}
\caption{Performance of \texttt{KQ} and \texttt{SMC-KQ} on the integration problem with $f(\bm{x})=1 + \prod_{j=1}^d \sin(2\pi x_j)$ integrated against a $\mathrm{N}(\bm{0},\mathbf{I})$ distribution for $d=2$ (top), $d=3$ (middle) and $d=10$ (bottom). The SMC sampler was initiated with a $\mathrm{N}(\bm{0},8^2\mathbf{I})$ distribution. 
The kernel used was a (multivariate) Gaussian kernel $k(\bm{x},\bm{y})=\exp(-\sum_{j=1}^d (x_j - y_j)^2/\ell_j^2)$ with the length scales $\ell_1= \dots = \ell_d = 0.25$ were used. 
}
\label{fig:larger_d}
\end{figure}

\begin{figure}[t!]
\begin{center}
\includegraphics[width = \columnwidth]{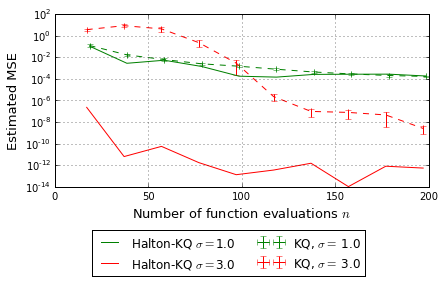}
\end{center}
\caption{Comparison between KQ with $x_j \sim \mathrm{N}(0,1)$ independent and KQ with $x_j = \Phi^{-1}(u_j)$ where the $\{u_j\}_{j=1}^n$ are the first $n$ terms in the Halton sequence and $\Phi$ is the standard Gaussian cumulative density function.}
\label{fig:KQ-Halton}
\end{figure}

\newpage

\subsubsection{Implementation of Stein's Method}

Following \citet{Oates2017} we considered the Stein operator 
$$
\mathbb{S}[f](\bm{\theta}) :=  [\nabla_{\bm{\theta}} + \nabla \log \pi(\bm{\theta})][f](\bm{\theta})
$$ 
and denote the score function by $u_j(\bm{\theta})=\nabla_{\theta_j}\log\pi(\bm{\theta})$. 
Here $\pi$ is the p.d.f. for $\Pi$.
Applying the Stein operator to each argument of a base kernel $k_b$, and adding a constant, gives produces the new kernel:
\begin{eqnarray*}
k(\bm{\theta},\bm{\phi}) 
& := & 1 + \sum_{j=1}^d \begin{array}{l} [ \nabla_{\theta_j} \nabla_{\phi_j} k_b(\bm{\theta},\bm{\phi}) \\
\hspace{20pt} + u_j(\bm{\theta}) \nabla_{\phi_j} k_b(\bm{\theta},\bm{\phi}) \\
\hspace{20pt} + u_j(\bm{\phi}) \nabla_{\theta_j} k_b(\bm{\theta},\bm{\phi})  \\
\hspace{20pt} + u_j(\bm{\theta}) u_j(\bm{\phi}) k_b(\bm{\theta},\bm{\phi}) ]
\end{array}
\end{eqnarray*}
which we will use for our KQ estimator. Using integration by parts, we can easily check that $\Pi[k(\cdot,\bm{\theta})] = 1$ and $\Pi \otimes \Pi(k) = 1$. 
In this experiment, the base kernel was taken to be Gaussian: $k_b(\bm{\theta},\bm{\phi})=\exp(-\sum_{j=1}^d (\theta_j-\phi_j)^2/\ell_j^2)$.
We obtained the derivatives:
\begin{eqnarray*}
\frac{\mathrm{d}k(\bm{\theta},\bm{\phi})}{\mathrm{d}\theta_j}  & = & -\frac{2}{\ell_j^2} (\theta_j - \phi_j) k(\bm{\theta},\bm{\phi}) \\
\frac{\mathrm{d}k(\bm{\theta},\bm{\phi})}{\mathrm{d}\phi_j}  & = &  \frac{2}{\ell_j^2} (\theta_j - \phi_j) k(\bm{\theta},\bm{\phi}) \\
\frac{\mathrm{d} k(\bm{\theta},\bm{\phi})}{\mathrm{d} \theta_j \mathrm{d} \phi_j}  & = &  \frac{\big( 2\ell_j^2 - 4 (\theta_j - \phi_j)^2\big)}{\ell_j^4} k(\bm{\theta},\bm{\phi})  \\
\end{eqnarray*}
Furthermore, we can obtain expressions for the score function for posterior densities as follows:
\begin{eqnarray*}
u_j(\bm{\theta}) 
& = & \frac{\mathrm{d}}{\mathrm{d}\theta_j} \log\pi(\bm{\theta}) + \frac{\mathrm{d}}{\mathrm{d}\theta_j} \log\pi(\bm{y}|\bm{\theta}).
\end{eqnarray*}

\subsection{Algorithms and Implementation} \label{appendix:algorithms_pseudocode}

\subsubsection{SMC Sampler}

In Alg. \ref{SMC} the standard SMC scheme is presented.
Re-sampling occurs when the effective sample size, $\|\bm{w}\|_2^{-2}$ drops below a fraction $\rho$ of the total number $N$ of particles.
In this work we took $\rho = 0.95$ which is a common default.

\begin{algorithm}[h!]
\caption{Sequential Monte Carlo Iteration} 
\label{SMC}
\begin{algorithmic}
\STATE {\bf function} $\texttt{SMC}(\{(w_j,\bm{x}_j)\}_{j=1}^N , t_i , t_{i-1} , \rho)$
\STATE {\bf input} $\{(w_j , \bm{x}_j)\}_{j=1}^N$ (particle approx. to $\Pi_{i-1}$)
\STATE {\bf input} $t_i$ (next inverse-temperature)
\STATE {\bf input} $t_{i-1}$ (previous inverse-temperature)
\STATE {\bf input} $\rho$ (re-sample threshold)
\STATE $w_j' \gets w_j \times [ \pi(\bm{x}_j) / \pi_0(\bm{x}_j) ]^{t_i - t_{i-1}}$ ($\forall j \in 1:N$)
\STATE $\bm{w}' \gets \bm{w}' / \|\bm{w}'\|_1$ (normalise weights)
\IF {$\|\bm{w}'\|_2^{-2} < N \cdot \rho$}
\STATE $\bm{a} \sim \text{Multinom}(\bm{w}')$
\STATE $\bm{x}_j' \gets \bm{x}_{a(j)}$ (re-sample $\forall j \in 1:N$)
\STATE $w_j' \gets N^{-1}$ (reset weights $\forall j \in 1:N$)
\ENDIF
\STATE $\bm{x}_j' \sim \texttt{Markov}(\bm{x}_j'; \Pi_i , \{(w_j,\bm{x}_j)\}_{j=1}^N)$ (Markov update $ \in 1:N$)
\STATE {\bf return} $\{(w_j',\bm{x}_j')\}_{j=1}^N$ (particle approx. to $\Pi_i$)
\end{algorithmic}
\end{algorithm} \FloatBarrier

Denote
\begin{eqnarray*}
q(\bm{x},\cdot;\{(w_j,\bm{x}_j)\}_{j=1}^N) & = & \mathrm{N}(\cdot ; \bm{\mu} , \bm{\Sigma} ) \\
\bm{\mu} & = & \sum_{j=1}^N w_j \bm{x}_j \\
\bm{\Sigma} & = & \sum_{j=1}^N w_j (\bm{x}_j - \bm{\mu}) (\bm{x}_j - \bm{\mu})^\top.
\end{eqnarray*}
The above standard adaptive independence proposal was used within a Metropolis-Hastings Markov transition:

\begin{algorithm}[h!]
\begin{algorithmic}
\caption{Markov Iteration} 
\label{markov}
\STATE {\bf function} $\texttt{Markov}(\bm{x} , \pi, \{(w_j,\bm{x}_j)\}_{j=1}^N)$
\STATE {\bf input} $\bm{x}$ (current state)
\STATE {\bf input} $\pi$ (density of invar. dist.)
\STATE $\bm{x}^* \sim q(\bm{x},\bm{x}^*;\{(w_j,\bm{x}_j)\}_{j=1}^N)$ (propose)
\STATE $$\hspace{-75pt} r \gets \frac{ \pi_i(\bm{x}^*) q(\bm{x}^*,\bm{x}; \{(w_j,\bm{x}_j)\}_{j=1}^N) }{  \pi_i(\bm{x}) q(\bm{x},\bm{x}^*; \{(w_j,\bm{x}_j)\}_{j=1}^N) }$$
\STATE $u \sim \text{Unif}(0,1)$
\IF {$u < r$}
\STATE $\bm{x} \gets \bm{x}^*$ (accept)
\ENDIF 
{\bf return} $\bm{x}$ (next state)
\end{algorithmic}
\end{algorithm} \FloatBarrier

\subsubsection{Choice of Temperature Schedule}

Following \citet{Zhou2016} we employed an adaptive temperature schedule construction.
This was based on the conditional effective sample size of the SMC particle set, estimated as follows:

\begin{algorithm}[h!]
\begin{algorithmic}
\caption{Conditional Effective Sample Size} 
\label{cess}
\STATE {\bf function} $\texttt{CESS}(\{(w_j , \bm{x}_j)\}_{j=1}^N , t)$
\STATE {\bf input} $\{(w_j , \bm{x}_j)\}_{j=1}^N$ (particle approx. $\Pi_{i-1}$)
\STATE {\bf input} $t$ (candidate next inverse-temperature) 
\STATE $z_j \gets [\pi(\bm{x}_j) / \pi_0(\bm{x}_j)]^{t_i - t_{i-1}}$ ($\forall j \in 1:N$)
\STATE $E \gets N \left( \sum_{j=1}^N w_j z_j \right)^2 \big/  \left( \sum_{j=1}^N w_j z_j^2 \right)$ 
\STATE {\bf return} $E$ (est'd. cond. ESS)
\end{algorithmic}
\end{algorithm} \FloatBarrier

The specific construction for the temperature schedule is detailed in Alg. \ref{temp} below and makes use of a Sequential Least Squares Programming algorithm:

\begin{algorithm}[h!]
\caption{Adaptive Temperature Iteration} 
\label{temp}
\begin{algorithmic}
\STATE {\bf function} $\texttt{temp}(\{(w_j , \bm{x}_j)\}_{j=1}^N , t_{i-1} , \rho,\Delta)$
\STATE {\bf input} $\{(w_j , \bm{x}_j)\}_{j=1}^N$ (particle approx. $\Pi_{i-1}$)
\STATE {\bf input} $t_{i-1}$ (current inverse-temperature)
\STATE {\bf input} $\rho$ (re-sample threshold)   
\STATE {\bf input} $\Delta$
(max. grid size, default $\Delta = 0.1$)  
\STATE $t \gets \text{solve}(\texttt{CESS}(\{(w_j , \bm{x}_j)\}_{j=1}^N , t) = N \cdot \rho)$ \\ (binary search in $[t_{i-1},1]$)
\STATE $t_i \gets \min\{t_{i-1} + \Delta , t\}$
{\bf return} $t_i$ (next inverse-temperature)
\end{algorithmic}
\end{algorithm} \FloatBarrier

\subsubsection{Termination Criterion}

For \texttt{SMC-KQ} we estimated an upper bound on the worst case error in the unit ball of the Hilbert space $\mathcal{H}$.
This was computed as follows, using a bootstrap algorithm:

\begin{algorithm}[h!]
\caption{Termination Criterion} 
\label{term}
\begin{algorithmic}
\STATE {\bf function} $\texttt{crit}(\Pi , k , \{\bm{x}_j\}_{j=1}^N)$
\STATE {\bf input} $\Pi$ (target disn.)
\STATE {\bf input} $k$ (kernel)
\STATE {\bf input} $\{\bm{x}_j\}_{j=1}^N$ (collection of states)
\STATE $R^2 \gets 0$
\STATE $e_0 \gets \iint_{\mathcal{X} \times \mathcal{X}} k(\bm{x},\bm{x}') \Pi \otimes \Pi(\mathrm{d}\bm{x} \times \mathrm{d}\bm{x}')$ (in'l error)
\FOR{m = 1,\dots,M} 
\STATE $\tilde{\bm{x}}_j \sim \text{Unif}(\{\bm{x}_j\}_{j=1}^N)$ ($\forall j \in 1:n$)
\STATE $z_j \gets \int_{\mathcal{X}} k(\cdot , \tilde{\bm{x}}_j) \mathrm{d}\Pi$ (k'l mean eval. $\forall j \in 1:n$)
\STATE $\mathrm{K}_{j,j'} \gets k(\tilde{\bm{x}}_j , \tilde{\bm{x}}_{j'})$ (kernel eval. $\forall j,j' \in 1:n$)
\STATE $\bm{w} \gets \bm{z}^T\bm{K}^{-1}$ (KQ weights)
\STATE $e_n^2 \gets \bm{w}^\top \bm{\mathrm{K}} \bm{w} - 2 \bm{w}^\top \bm{z} + e_0^2$
\STATE $R^2 \gets R^2 + e_n^2 M^{-1}$ 
\ENDFOR \\
{\bf return} $R$ (est'd error)
\end{algorithmic}
\end{algorithm} \FloatBarrier
Note that this could be slightly improved using a weighted bootstrap approach.

For \texttt{SMC-KQ-KL} an empirical upper bound on integration error was estimated.
This requires that the norm $\|f\|_{\mathcal{H}}$ be estimated, which was achieved as follows:

\begin{algorithm}[h!]
\caption{Termination Crit. + Kernel Learning} 
\label{term kl}
\begin{algorithmic}
\STATE {\bf function} $\texttt{crit-KL}(f, \Pi , k , \{\bm{x}_j\}_{j=1}^N)$
\STATE {\bf input} $f$ (integrand)
\STATE {\bf input} $\Pi$ (target disn.)
\STATE {\bf input} $k$ (kernel)
\STATE {\bf input} $\{\bm{x}_j\}_{j=1}^N$ (collection of states) 
\STATE $R^2 \gets 0$
\STATE $e_0 \gets \iint_{\mathcal{X} \times \mathcal{X}} k(\bm{x},\bm{x}') \Pi \otimes \Pi(\mathrm{d}\bm{x} \times \mathrm{d}\bm{x}')$ (in'l error)
\FOR{m = 1,\dots,M} 
\STATE $\tilde{\bm{x}}_j \sim \text{Unif}(\{\bm{x}_j\}_{j=1}^N)$ ($\forall j \in 1:n$)
\STATE $\mathrm{f}_j \gets f(\tilde{\bm{x}}_j)$ (function eval. $\forall j \in 1:n$)
\STATE $z_j \gets \int_{\mathcal{X}} k(\cdot , \tilde{\bm{x}}_j) \mathrm{d}\Pi$ (k'l mean eval. $\forall j \in 1:n$)
\STATE $\mathrm{K}_{j,j'} \gets k(\tilde{\bm{x}}_j , \tilde{\bm{x}}_{j'})$ (kernel eval. $\forall j,j' \in 1:n$)
\STATE $\bm{w} \gets \bm{z}^T\bm{K}^{-1}$ (KQ weights)
\STATE $e_n^2 \gets \bm{w}^\top \bm{\mathrm{K}} \bm{w} - 2 \bm{w}^\top \bm{z} + e_0^2$
\STATE $R^2 \gets R^2 + e_n^2 M^{-1}$ 
\ENDFOR
\STATE $z_j \gets \int_{\mathcal{X}} k(\cdot , \bm{x}_j) \mathrm{d}\Pi$ (kernel mean eval. $\forall j \in 1:n$)
\STATE $\mathrm{K}_{j,j'} \gets k(\bm{x}_j ,\bm{x}_{j'})$ (kernel eval. $\forall j,j' \in 1:n$)
\STATE $\bm{w} \gets \bm{z}^T\bm{K}^{-1}$ (KQ weights)
\STATE $S^2 \gets R^2 \times \bm{w}^\top \bm{\mathrm{K}} \bm{w}$ 
{\bf return} $S$ (est'd error bound)
\end{algorithmic}
\end{algorithm} \FloatBarrier

In Alg. \ref{term kl} the literal interpretation, that $f$ is re-evaluated on values of $\bm{x}_j$ which have been previously examined, is clearly inefficient.
In practice such function evaluations were cached and then do not contribute further to the total number of function evaluations that are required in the algorithm.

\subsubsection{Kernel Learning}

A generic approach to select kernel parameters is the maximum marginal likelihood method:

\begin{algorithm}[h!]
\caption{Parameter Update} 
\label{kern-param}
\begin{algorithmic}
\STATE {\bf function} $\texttt{kern-param}(\bm{\mathrm{f}} , \{\bm{x}_j\}_{j=1}^n  , k_\theta)$
\STATE {\bf input} $\bm{\mathrm{f}}$ (integrand evals.)
\STATE {\bf input} $\{\bm{x}_j\}_{j=1}^n$ (associated states)
\STATE {\bf input} $k_\theta$ (parametric kernel) 
\STATE $\theta' \gets \argmin_\theta \bm{\mathrm{f}}^\top \bm{\mathrm{K}}_\theta^{-1} \bm{\mathrm{f}} + \log |\bm{\mathrm{K}}_\theta|$ (numer. opt.)
\STATE (s.t. $\mathrm{K}_{\theta,j,j'} = k_\theta(\bm{x}_j , \bm{x}_{j'})$) 
{\bf return} $\theta'$ (optimal params)
\end{algorithmic}
\end{algorithm} \FloatBarrier

\subsubsection{Implementation of \texttt{SMC-KQ-KL}}

Our final algorithm to present is the full implementation for \texttt{SMC-KQ-KL}:

\begin{algorithm}[t!]
\caption{SMC for KQ with Kernel Learning}
\label{SMCKQKL}
\begin{algorithmic}
\STATE {\bf function} $\texttt{SMC-KQ-KL}(f, \Pi, k_\theta , \Pi_0, \rho , n, N)$
\STATE {\bf input} $f$ (integrand)
\STATE {\bf input} $\Pi$ (target disn.)
\STATE {\bf input} $k_\theta$ (parametric kernel)
\STATE {\bf input} $\Pi_0$ (reference disn.)
\STATE {\bf input} $\rho$ (re-sample threshold)
\STATE {\bf input} $n$ (num. func. evaluations)
\STATE {\bf input} $N$ (num. particles) 
\STATE $i \gets 0$; $t_i \gets 0$; $R_{\min} \gets \infty$
\STATE $\bm{x}_j' \sim \Pi_0$ (initialise states $\forall j \in 1:N$)
\STATE $w_j' \gets N^{-1}$ (initialise weights $\forall j \in 1:N$)
\STATE $\theta' \gets \texttt{kern-param}(f,\{\bm{x}_j'\}_{j=1}^n)$ (kernel params)
\STATE $R \gets \texttt{crit-KL} (f, \Pi , k_{\theta'} , \{\bm{x}_j'\}_{j=1}^N)$ (est'd error)
\WHILE{$\texttt{test}(R < R_{\min})$ and $t_i < 1$}
\STATE $i \gets i + 1$; $R_{\min} \gets R$; $\theta \gets \theta'$
\STATE $\{(w_j , \bm{x}_j)\}_{j=1}^N \gets \{(w_j' , \bm{x}_j')\}_{j=1}^N$
\STATE $t_i \gets \texttt{temp}(\{(w_j , \bm{x}_j)\}_{j=1}^N , t_{i-1})$ (next temp.)
\STATE $\{(w_j' , \bm{x}_j')\}_{j=1}^N \gets \texttt{SMC}(\{(w_j,\bm{x}_j)\}_{j=1}^N , t_i , t_{i-1} , \rho)$ \\ (next particle approx.)
\STATE $\theta' \gets \texttt{kern-param}(f,\{\bm{x}_j'\}_{j=1}^n)$ (kernel params)
\STATE $R \gets \texttt{crit-KL} (f, \Pi , k_{\theta'} , \{\bm{x}_j'\}_{j=1}^N)$ (est'd error)
\ENDWHILE
\STATE $\mathrm{f}_j \gets f(\bm{x}_j)$ (function eval. $\forall j \in 1:n$)
\STATE $z_j \gets \int_{\mathcal{X}} k_\theta(\cdot , \bm{x}_j) \mathrm{d}\Pi$ (kernel mean eval. $\forall j \in 1:n$)
\STATE $\mathrm{K}_{j,j'} \gets k_\theta(\bm{x}_j , \bm{x}_{j'})$ (kernel eval. $\forall j,j' \in 1:n$)
\STATE $\hat{\Pi}(f) \gets \bm{z}^\top \bm{\mathrm{K}}^{-1} \bm{\mathrm{f}}$ (eval. KQ estimator)
{\bf return} $\hat{\Pi}(f)$ (estimator) 
\end{algorithmic}
\end{algorithm}  \FloatBarrier

As stated here, Alg. \ref{SMCKQKL} is inefficient as function evaluations that are produced in the $\texttt{kern-param}$ and $\texttt{crit-KL}$ components are not included in the KQ estimator $\hat{\Pi}(f)$.
Thus a trivial modification is to store all function evaluations $(\mathrm{f}_j , \bm{x}_j)$ that are produced and to include all of these in the ultimate KQ estimator.
This was the approach taken in our experiments that involved \texttt{SMC-KQ-KL}.
However, since it is somewhat cumbersome to include in the pseudo-code, we have not made this explicit in the notation.
Our reported results are on a per-function-evaluation basis and so we {\bf do} adjust for this detail in our reported comparisons.

\end{document}